%% file: main.tex
\documentclass{article}

\usepackage[final]{neurips_2021}

\usepackage[utf8]{inputenc} % allow utf-8 input
\usepackage[T1]{fontenc}    % use 8-bit T1 fonts
\usepackage[hidelinks]{hyperref}       % hyperlinks
\usepackage{url}            % simple URL typesetting
\usepackage{booktabs}       % professional-quality tables
\usepackage{amsfonts}       % blackboard math symbols
\usepackage{nicefrac}       % compact symbols for 1/2, etc.
\usepackage{microtype}      % microtypography
\usepackage{xcolor}         % colors

\title{Rethinking Neural Operations for Diverse Tasks}

\newcommand\blankfootnote[1]{%
	\begingroup
	\renewcommand\thefootnote{}\footnote{#1}%
	\addtocounter{footnote}{-1}%
	\endgroup
}

\author{%
  Nicholas Roberts$^\ast$\\
  University of Wisconsin-Madison\\
  \texttt{nick11roberts@cs.wisc.edu}\\
  \And Mikhail Khodak$^\ast$\\
  Carnegie Mellon University\\
  \texttt{khodak@cmu.edu}\\
  \And Tri Dao\\
  Stanford University\\
  \texttt{trid@stanford.edu}\\
  \And Liam Li\\
  Hewlett Packard Enterprise\\
  \texttt{me@liamcli.com} \\
  \And Christopher R\'e\\
  Stanford University\\
  \texttt{chrismre@cs.stanford.edu}\\
  \And Ameet Talwalkar \\
  Carnegie Mellon University \& Hewlett Packard Enterprise\\
  \texttt{talwalkar@cmu.edu}
}

\usepackage{graphicx}
\usepackage{subfigure}

\usepackage{amsmath,amssymb,amsthm}
\usepackage[ruled,vlined]{algorithm2e}
\usepackage[utf8]{inputenc}

\newtheorem{Def}{Definition}[section]

\newtheorem{Clm}{Claim}[section]

\newtheorem{Set}{Setting}[section]
\newtheorem{Rem}{Remark}[section]

\newcommand\XD{\mathbf{XD}}

\DeclareMathOperator{\diag}{\operatorname{diag}}
\DeclareMathOperator{\Op}{\operatorname{\bf Op}}
\DeclareMathOperator{\Real}{\operatorname{Real}}
\newcommand{\BigO}{\mathcal O}
\DeclareMathOperator{\Conv}{\operatorname{\bf Conv}}
\DeclareMathOperator{\Lin}{\operatorname{\bf Lin}}
\DeclareMathOperator{\Id}{\operatorname{\bf Id}}
\DeclareMathOperator{\Zero}{\operatorname{\bf Zero}}
\DeclareMathOperator{\MaxP}{\operatorname{\bf MaxPool}}
\DeclareMathOperator{\AvgP}{\operatorname{\bf AvgPool}}
\DeclareMathOperator{\DilC}{\operatorname{\bf DilatedConv}}

\newcommand{\Search}{\mathcal S}

\newcommand{\disc}{\mathbf{discrete}}

\newcommand{\A}{\mathcal A}

\newcommand{\C}{\mathbb C}

\newcommand{\F}{\mathcal F}
\newcommand{\G}{\mathcal G}

\newcommand{\N}{\mathbb N}

\newcommand{\R}{\mathbb R}

\newcommand{\W}{\mathcal W}
\newcommand{\X}{\mathcal X}
\newcommand{\Y}{\mathcal Y}
\newcommand{\Z}{\mathbb Z}
\newcommand{\0}{\mathbf 0}
\newcommand{\1}{\mathbf 1}
\newcommand{\atr}{\mathfrak{a}}

\usepackage{scalerel,mathtools,color}
\let\svsqrt\sqrt
\newsavebox\Nsqrt
\def\sr#1{\ThisStyle{%
		\savebox\Nsqrt{\scalebox{.5}[1]{$\SavedStyle\svsqrt{\phantom{\cramped{#1#1}}}$}}%
		\ooalign{\usebox{\Nsqrt}\cr\kern.2pt\usebox{\Nsqrt}\cr\hfil$\SavedStyle\cramped{#1}$}}}

\usepackage{enumitem}
\usepackage{bm}
\def\*#1{\mathbf{#1}}

\usepackage{threeparttable,multirow,caption}

\begin{document}

\maketitle

\vspace{-5mm}
\blankfootnote{$\ast$ denotes equal contribution.}
\input{abstract}
\input{intro}
\input{related}
\input{operations}
\input{chrysalis}
\input{pde}
\input{protein}
\input{seq}
\input{conclusion}

\vspace{-1mm}
\section*{Acknowledgments}
\vspace{-1mm}

We thank Maria-Florina Balcan, Jeremy Cohen, and Tian Li for helpful advice on early versions of this paper and anonymous reviewers for suggested improvements.
This work was supported in part by DARPA under cooperative agreements FA875017C0141 and HR0011202000, NSF grants CCF-1535967, CCF-1910321, IIS-1618714, IIS-1705121, IIS-1838017, IIS-1901403, and IIS-2046613, a Microsoft Research Faculty Fellowship, a Bloomberg Data Science research grant, an Amazon Research Award, an AWS Machine Learning Research Award, a Facebook Faculty Research Award, funding from Booz Allen Hamilton Inc., a Block Center Grant, a Carnegie Bosch Institute Research Award, and a Two Sigma Fellowship Award.
We also gratefully acknowledge the support of NIH under No.\ U54EB020405 (Mobilize), NSF under Nos.\ CCF1763315 (Beyond Sparsity), CCF1563078 (Volume to Velocity), and 1937301 (RTML); ONR under No.\ N000141712266 (Unifying Weak Supervision); the Moore Foundation, NXP, Xilinx, LETI-CEA, Intel, IBM, Microsoft, NEC, Toshiba, TSMC, ARM, Hitachi, BASF, Accenture, Ericsson, Qualcomm, Analog Devices, the Okawa Foundation, American Family Insurance, Google Cloud, Swiss Re, Total, the HAI-AWS Cloud Credits for Research program, the Stanford Data Science Initiative (SDSI), and members of the Stanford DAWN project: Facebook, Google, and VMWare.
The Mobilize Center is a Biomedical Technology Resource Center, funded by the NIH National Institute of Biomedical Imaging and Bioengineering through Grant P41EB027060.
The U.S.\ Government is authorized to reproduce and distribute reprints for Governmental purposes notwithstanding any copyright notation thereon.
Any opinions, findings and conclusions, or recommendations expressed in this material are those of the authors and do not necessarily reflect the views of DARPA, NSF, NIH, ONR, or any other funding agency.

\bibliography{refs}
\bibliographystyle{plain}

\appendix
\input{expressivity}
\input{complexity}
\input{cifar}
\input{apppde}
\input{approtein}
\input{appseq}

\end{document}

%% file: abstract.tex
% !TEX root = main.tex

\begin{abstract}

	%An important goal of neural architecture search (NAS) is to automate-away the design of neural networks on new tasks in under-explored domains. 
	An important goal of AutoML is to automate-away the design of neural networks on new tasks in under-explored domains.
	Motivated by this goal, we study the problem of enabling users to discover the right neural operations given data from their specific domain.
	We introduce a search space of operations called XD-Operations that mimic the inductive bias of standard multi-channel convolutions while being much more expressive:
	we prove that it includes many named operations across multiple application areas.
	Starting with any standard backbone such as ResNet, we show how to transform it into a search space over XD-operations and how to traverse the space using a simple weight-sharing scheme.
	On a diverse set of tasks---solving PDEs, distance prediction for protein folding, and music modeling---our approach consistently yields models with lower error than baseline networks and often even lower error than expert-designed domain-specific approaches.
	
\end{abstract}

%% file: intro.tex
% !TEX root = main.tex

\vspace{-1mm}
\section{Introduction}
\vspace{-1mm}

Automated machine learning (AutoML) and neural architecture search (NAS) are often motivated by a vision of democratizing ML by reducing the need for expert design on a variety of tasks.
While NAS has grown rapidly with developments such as weight-sharing~\citep{pham2018enas} and ``NAS-benches'' \citep{ying2019nasbench101,zela2020nasbench1shot1}, most efforts focus on search spaces that glue together established primitives for well-studied tasks like vision and text \citep{liu2019darts,li2019rsws,xu2020pcdarts,li2021gaea} or on issues such as latency \citep{cai2020ofa,fang2020densenas}.
In this work, we revisit the broader vision of NAS and propose to move towards much more general search spaces while still exploiting successful network topologies.
To do so we focus on expanding the set of operations, which is usually fairly small;
for example, that of the well-studied DARTS space has eight elements: a few types of convolution and pooling layers \citep{liu2019darts}.
The baseline approach for expanding this set---adding operations one-by-one---scales poorly and will not result in new operations when faced with new types of data.

%Our core contribution is a re-imagining of NAS operation spaces inspired by {\em weight-sharing} \citep{pham2018enas}, a popular search approach in which a gradient algorithm optimizes a ``supernet" objective that interpolates the set of operations on each edge using a continuous relaxation.
%While most past work simply uses a convex combination of operations \citep{liu2019darts}, our search space is based upon an alternative continuous relaxation that exploits the fact that most of these operations return linear transforms diagonalized by the discrete Fourier transform (DFT).
%Replacing the DFT matrices in the diagonal decomposition by a more expressive family of efficient linear transforms known as {\em Kaleidoscope} or {\em K-matrices} \citep{dao2020kaleidoscope} yields the set of {\bf Expressive Diagonalization (XD) Operations}, which comprise a large search space containing various types of grid-based convolutions and pooling, permutations, certain kinds of graph convolutions, the Fourier Neural Operator (FNO) from the PDE literature \citep{li2021fno}, and infinitely many more.

Our core contribution is a re-imagining of NAS operation spaces that drastically expands this set in a principled fashion to include both standard operations as well as a wide range of new ones. 
To do so we exploit the fact that most standard operations used in modern NAS return linear transforms diagonalized by the discrete Fourier transform (DFT). Replacing the DFT matrices in the diagonal decomposition by a more expressive family of efficient linear transforms known as {\em Kaleidoscope} or {\em K-matrices} \citep{dao2020kaleidoscope} yields the set of {\bf Expressive Diagonalization (XD) Operations}, which comprise a large search space containing various types of grid-based convolutions and pooling, permutations, transposed convolutions, certain kinds of graph convolutions, the Fourier Neural Operator (FNO)~\citep{li2021fno}, and infinitely many more.
This broad expressivity reflects the key insight of our work:
that many of the most important neural operations in ML consist of multiple channels that apply weights $\*w$ to inputs $\*x$ by computing
\begin{equation}
\*K\diag(\*L\*w)\*M\*x
\end{equation}
where the matrices $\*K$, $\*L$, and $\*M$ are {\em efficient} (to represent and apply) and shared across channels.
%By design, all XD-operations are also computationally efficient and have small description length.

\begin{figure*}[t!]
	\centering
		\includegraphics[width=\linewidth]{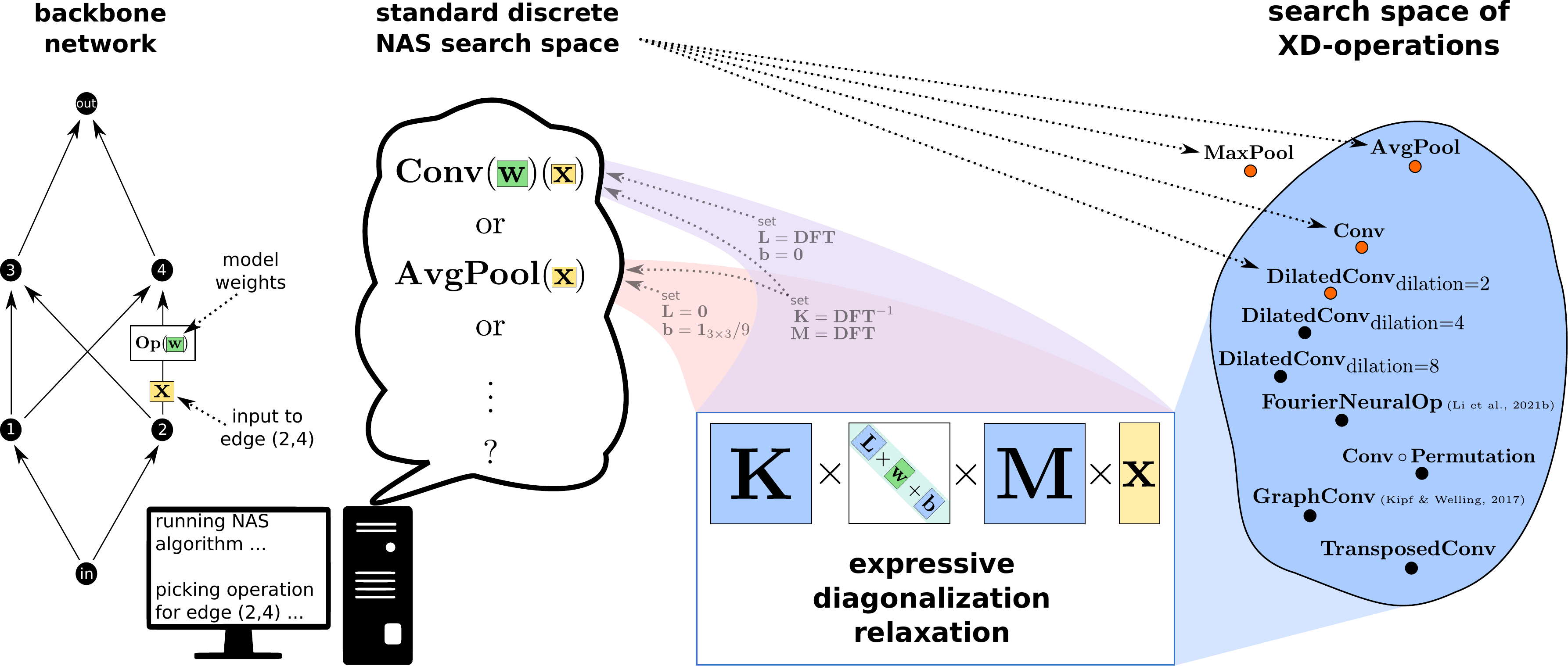}
	\caption{\label{fig:main}
		Diagram of our search space depicting a NAS method picking an operation for an edge in a backbone network (left).
		Instead of choosing from a discrete search space, we use a relaxation based on the convolution's diagonalization by the discrete Fourier transform in which the DFTs are replaced by K-matrices \citep{dao2020kaleidoscope} $\*K$, $\*L$, and $\*M$ (middle);
		these are the main architecture parameters of our new search space over Expressive Diagonalization (XD) operations.
		This space contains most operations considered in standard NAS and many other important operations in a variety of domains (right).
	}
\end{figure*}

We leverage XD-operations to take critical steps towards a broader NAS that enables the discovery of good design patterns with limited human specification from data in under-explored domains.
To do so we develop a simple procedure which transforms any backbone convolutional neural network (CNN) into an architecture search space by replacing its operations with XD-operations.
This space is then searched using a simple weight-sharing algorithm that needs only a small amount of tuning to find effective operations.
As a simple first demonstration, we show that XD-operations yield models that are 15\% more accurate than standard discrete search spaces on {\em permuted} CIFAR-10, highlighting the fragility of standard NAS operation spaces on new datasets, and thus the need for XD-operations.

As our main evaluation, we demonstrate the effectiveness of XD-operations in a series of applications showing that, starting from vanilla CNNs, they consistently outperform custom-designed operations.
\begin{itemize}[leftmargin=*,topsep=-1pt,noitemsep]\setlength\itemsep{2pt}
%	\item We first demonstrate the fragility of NAS operation spaces---and thus the need for XD-operations---with a simple experiment using CIFAR-10.
%	While model found using XD-operations are 
%	While searching over XD-operations obtains  comparable to search over DARTS operations yields comparable performance
%%	using the LeNet and ResNet-20 backbones 
%	on the original data, models found using XD-operations are 15\% more accurate when the images are permuted.
%	In this setting, we even exceed a state-of-the-art NAS algorithm over the DARTS search space by 7\% while using a much smaller model. \misha{Where to include this bullet?}
	\item {\bf Learning to solve partial differential equations (PDEs):} when substituted into a simple CNN backbone, XD-operations outperform convolutions and the dense prediction NAS method Auto-DeepLab \citep{liu2019autodl}, and even achieve lower error than custom-designed, state-of-the-art operations (FNOs \citep{li2021fno}) across three problems with different dimensionalities (Burgers' equation, Darcy Flow, and Navier-Stokes).
	Our method also maintains consistent performance across different resolutions, a major stated advantage of FNOs over previous methods.
	\item {\bf Protein folding:}
	on the task of predicting residue distances in a polypeptide chain---a key component of the protein folding problem---we substitute XD-operations into vanilla ResNets and achieve lower error than cyclically-dilated ResNets adapted specifically for this setting \citep{adhikari2020}.
	Furthermore, our ResNet-34 XD outperforms the reported error of the much deeper Dilated ResNet-258.
	\item {\bf Music modeling:} on two next-note prediction tasks, we show that substituting XD-operations into an undilated CNN  outperforms temporal convolutional networks (TCNs)---exponentially-dilated 1d CNNs that themselves outperform standard convolutional and recurrent networks \citep{bai2018tcn}.
%	We study four problems in sequence modeling using temporal convolutional networks (TCNs) \cite{bai2018tcn} as backbones.
%	By substituting XD-operations we exceed this strong baseline, itself usually better than classical recurrent nets such as LSTMs, on all tasks, including by four perplexity points on word-level Penn Treebank.
\end{itemize}

Code to reproduce these results is available here: \url{https://github.com/nick11roberts/XD}.
Software to apply XD-operations can be found here: \url{https://github.com/mkhodak/relax}.

%
%\begin{algorithm}[!h]
%	\DontPrintSemicolon
%	\KwIn{input to algorithm}
%	\For{iteration}{
%		do something \tcp*{comment}
%	}
%	\caption{\label{alg:example}
%		algorithm description
%	}
%\end{algorithm}
%
%\begin{table}[!h]
%	\centering
%	\begin{threeparttable}
%		\begin{tabular}{lccccccc}
%			\hline
%			& & \multicolumn{2}{c}{heading}& heading  & heading & heading & heading \\  
%			heading & heading & subheading & subheading & subheading & subheading & subheading & subheading \\
%			\hline
%			section 1$^\ast$ & - & - & - & - & - & - & - \\
%			\hline
%			section 2$^\dagger$ & - & - & - & - & - & - & - \\
%			\hline
%		\end{tabular}
%		\begin{tablenotes}
%			\item[$\ast$] table note 1
%			\item[$\dagger$] table note 2
%		\end{tablenotes}
%		\caption{\label{tab:example}
%			table description
%		}
%	\end{threeparttable}
%\end{table}

%% file: related.tex
% !TEX root = main.tex

\paragraph{Related Work}
AutoML is a well-studied area, with most work focusing on fairly small hyperparameter spaces~\citep{bergstra2012rs,li2018hyperband} or on NAS \citep{elsken2019nas}.
Most NAS operation spaces only contain a few operations such as convolutions \citep{liu2019darts,mei2020atomnas,zela2020nasbench1shot1,dong2020nasbench201}, which may not be useful for domains where CNNs are ineffective.
Applications of NAS outside vision largely follow the same pattern of combining human-designed operations \citep{nekrasov2019nas,wang2020textnas}. 
On the other extreme, AutoML-Zero \citep{real2020automlzero} demonstrates the possibility of evolving all aspects of ML from scratch.
We seek to establish a middle ground with large and domain-agnostic search spaces that still allow the use of well-tested methods, e.g. stochastic gradient descent (SGD).
%In the latter case, it has been observed that the search spaces are still ``easy" in the sense that random architectures do reasonably well \citep{elsken2019nas,li2019rsws}.
%Our main contribution is a family of search spaces generalizing convolutions \citep{lecun1999lenet}.

Several papers have generalized the DFT to replace layers in deep nets \citep{dao2019butterfly,alizadeh2020butterfly,ailon2020butterfly,dao2020kaleidoscope} in order to speed up or add structure to models while {\em reducing} expressivity.
In contrast, we can replace {\em convolutions} and other layers while {\em increasing} expressivity by extending their diagonalization via K-matrices.
As discussed in Section~\ref{sec:relax}, using K-matrices for this directly is inefficient for input dimension $>1$.

%Beyond NAS, recent work by \citet{zhou2021symmetries} uses a meta-learning framing \citep{thrun1998ltl} to study how to learn more general types of symmetries---beyond simply translational---from multi-task data.
%This transfer-based setup allows a clear formalization of learning such equivariances, though unlike NAS, it is not applicable to single-task settings.
%In addition, their technical approach does not generalize 2d convolutions due to computational intractability, while our XD-operations are indeed able to do so.

%Recently, \citet{neyshabur2020convolutions} showed that sparsity-inducing optimization can train fully connected nets that match the performance of convolutional networks and in the process the weights learn local connectivity patterns.
%However, none of these papers return parameterizable operations from a formally defined search space.

%% file: operations.tex
% !TEX root = main.tex

\section{The Expressive Diagonalization Relaxation}\label{sec:relax}
\vspace{-1mm}

In this section we overview our main contribution:
a large, general search space of neural operations.
Formally, we view an architecture as a {\em parameterizable} object---a mapping from model weights to functions---described by a {\em labeled} directed acyclic graph (DAG) $\G(V,E)$.
Each edge in $E$ has the form $(u,v,\Op)$, where $u,v\in V$ are nodes and $\Op$ is an operation that can be parameterized to define some transformation of the representation at node $u$;
node $v$ aggregates the outputs of its incoming edges into a new representation.
For example, the popular ResNet architecture \citep{he2016resnet} has many nodes with two incoming edges, one labeled by the convolution operation $\Conv$ and one by the identity (skip-connect) $\Id$, whose outputs it sums and passes to outgoing edges with the same labels.
Each architecture has a source node taking in input data and an output node returning a prediction.

Neural architecture search is the problem of automatically selecting an operation for each edge of $\G$ to optimize an objective.\footnote{It is often defined as selecting both operations and a graph topology \citep{zoph2018nas}, but if the set of operations contains the zero-operation $\Zero$ then the former subsumes the latter.}
For each edge $e\in E$ a NAS algorithm must pick one element of a {\em search space} $\Search=\{\Op_a\vert a\in\A\}$ of operations specified by architecture parameters $a\in\A$ to assign to $e$;
in past work, $\A$ usually indexes a small set of operations.
As an example, we will refer to a variant\footnote{For memory-efficiency, all convolutions in the original DARTS search space are separable \citep{liu2019darts}.} $\Search_\disc$ of the DARTS search space with parameters $\A_\disc=\{1,\dots,8\}$ where each operation is one of $\Zero$, $\Id$, $\MaxP_{3\times 3}$, $\AvgP_{3\times 3}$, $\Conv_{3\times 3\textrm{ or }5\times 5}$, or $\DilC_{3\times 3,2\textrm{ or }5\times 5,2}$~\citep{liu2019darts}.

Our main contribution is a novel family of operations that comprise a search space containing almost all these operations, in addition to many others that have been found useful on different types of data.
The starting point of our construction of these XD-operations is the simple observation that all the operations $\Op\in\Search_\disc$ listed above except $\MaxP_{3\times3}$ are {\em linear}, i.e.
for any model weights $\*w$ there exists a matrix $\*A_{\*w}$ such that for all inputs $\*x$ we have $\Op(\*w)(\*x)=\*A_{\*w}\*x$.
More specifically, all seven of them return convolutions:
to see this note that $\Zero$, $\Id$, and $\AvgP_{3\times3}$ each apply a convolution with filter $\0_{1\times1}$, $\1_{1\times1}$, and $\1_{3\times3}/9$, respectively.
This means that most of the operations in the DARTS search space---which is representative of NAS operation spaces in computer vision---share the convolution's diagonalization by the discrete Fourier transform (DFT).
Formally, if $\*A_{\*w}\in\R^{n^2\times n^2}$ is the matrix representing a 2d convolution with filter $\*w\in\R^{\*k}$ of kernel size $\*k\in[n]^2$, then for any 2d input $\*x\in\R^{n^2}$ we have
\begin{equation}\label{eq:fourier}
\Conv(\*w)(\*x)=\*A_{\*w}\*x=\*F^{-1}\diag\left(\*F\underline{\*w}\right)\*F\*x
\end{equation}
Here $[n]=\{1,\dots,n\}$, $\diag(\*z)$ denotes the diagonal matrix with entries $\*z$, $\underline{\*w}\in\R^{n^2}$ is an appropriate zero-padding of $\*w\in\R^{\*k}$, and $\*F\in\C^{n^2\times n^2}$ is the 2d DFT (a Kronecker product of two 1d DFTs).

This diagonalization explicates both the computational and representational efficiency of the DARTS operations, as the DFT and its inverse can be applied in time $\BigO(n\log n)$ and stored with $\BigO(n\log n)$ bits.
It also suggests a natural way to dramatically expand the operation space while preserving these efficiencies:
just replace matrices $\*F$ and $\*F^{-1}$ in \eqref{eq:fourier} by any one of a general family of efficient matrices.
Doing so yields the single-channel version of our {\em expressive diagonalization} (XD) operations:
\begin{equation}\label{eq:xd1}
\XD_\alpha^\1(\*w)(\*x)=\Real\left(\*K\diag\left(\*L\underline{\*w}\right)\*M\*x\right)
\end{equation}
Here architecture parameter $\alpha=(\*K,\*L,\*M)$ sets the matrices replacing $\*F$ and $\*F^{-1}$ in Equation~\ref{eq:fourier}.

The main remaining question is the family of efficient matrices to use, i.e. the domain of the architecture parameters $\*K$, $\*L$, and $\*M$.
For this we turn to the Kaleidoscope matrices, or {\em K-matrices}~\citep{dao2020kaleidoscope}, which generalize $\*F$ and $\*F^{-1}$ to include all computationally efficient linear transforms with short description length, including important examples such as sparse matrices and permutations.
To obtain this general family, K-matrices allow the DFT's butterfly factors---matrices whose products yield its efficient implementation---to take on different values.
While a detailed construction of K-matrices can be found in the original paper, we need only the following useful properties: 
they are as (asymptotically) efficient to apply as DFTs, are differentiable and can thus be updated using gradient-based methods, and can be composed (made ``deeper'') to make more expressive K-matrices.

Specifying that $\*K$, $\*L$, and $\*M$ in Equation~\ref{eq:xd1} are K-matrices largely completes our core contribution:
a new search space $\Search_\XD$ of XD-operations with K-matrix architecture parameters.
We give a full multi-channel formalization in $N$ dimensions, as well as an overview of its expressivity, in Section~\ref{sec:xd}.
First, we note some key aspects of this new search space:
\begin{itemize}[leftmargin=*,topsep=-1pt,noitemsep]\setlength\itemsep{2pt}
	\item{\bf Complexity:} $\XD_\alpha^\1(\*w)$ requires three K-matrices and $\BigO(1)$ filter weights to represent, i.e. description length $\BigO(n\log n)$;
	this is larger than a regular convolution (which has no architecture parameters) but is not quadratic in the input size like a linear layer.
	Applying $\XD_\alpha^\1$ requires multiplication by three K-matrices, yielding a theoretical per-channel time complexity of $\BigO(n\log n)$, matching the efficiency of convolutions.
	However, as XD-operations strictly generalize convolutions they are more expensive to apply in-practice;
	we detail these costs both in the application sections and as appendix table, and we view improving upon them as an important future direction.
	\item{\bf Initialization:} a crucial advantage of XD-operations is that we can initialize or {\em warm-start} search using operations with known constructions.
	In particular, since we can recover convolutions \eqref{eq:fourier} by setting architecture parameters $\*K=\*F^{-1}$, $\*L=\*F$, and $\*M=\*F$ in Equation~\ref{eq:xd1}, we can always start search with any CNN backbone.
	We use this extensively in experiments.
	\item{\bf K-matrices:} as they contain all efficient linear transforms, K-matrices can represent all functions returned by XD-operations, including convolutions.
	However, for input dimension and filter size $>1$ the only known way is to apply K-matrices directly to flattened inputs $\*x\in\R^{n^N}$, yielding much worse description length $\BigO(n^N\log n)$.
	In contrast, as detailed in Section~\ref{sec:xd}, our diagonalization approach uses Kronecker products to apply DFTs to each dimension separately, yielding description length $\BigO(n\log n)$.
	It is thus the first (and in some sense, ``right'') method to use such matrices to replace convolutions.
	Furthermore, diagonalization allows us to separate model weights $\*w$ from architecture parameters $\alpha$, letting the former vary across channels while fixing the latter.
\end{itemize}

Finally, we address the fact that the architecture parameters of $\Search_\XD$ are continuous, not discrete, contrasting with much of the NAS literature.
This can be viewed as a natural extension of the weight-sharing paradigm \citep{pham2018enas}, in which continuous relaxation enables updating architecture parameters with gradient methods.
For example, many algorithms traverse the relaxed DARTS search space 
$\tilde\Search_\disc=\left\{\sum_{i=1}^8\lambda_i\Op_i\vert\lambda_i\ge0,\sum_{i=1}^8\lambda_i=1\right\}$, defined via DARTS operations $\Op_i\in\Search_\disc$ and architecture parameters $\lambda_i$ in the 8-simplex;
most search spaces then require discretizing after search via a rounding procedure that maps from the simplex to $\A_\disc$.
%Our XD-relaxation avoids the poor scaling of the simplex relaxation when operations are added one-by-one.
%In particular, the per-iteration cost of many NAS algorithms increases linearly with operation count~\citep{liu2019darts} while the number of iterations of state-of-the-art methods increases logarithmically in the same~\citep{li2021gaea}.
%In contrast, $\Search_\XD$ contains numerous useful operations while taking roughly as long to search on CIFAR-10 as $\tilde\Search_\disc$.
Note that the fully continuous nature of XD-operations means that we will only evaluate the final network returned by search.
In particular, while some weight-sharing papers also report the correlation between true architecture performance and that indicated by the shared weights \citep{yang2020nas}, there is no obvious way to define a ranking or sampling distribution over XD-operations in order to do so.
This also means that our final architecture will not be more efficient than the supernet, unlike other weight-sharing methods that do discretize.

\vspace{-1mm}
\section{XD-Operations and Their Expressivity}\label{sec:xd}
\vspace{-1mm}

Here we formalize XD-operations and show what operations they include.
We first define operations:
\begin{Def}\label{def:po}
	A {\bf parameterizable operation} is a mapping $\Op:\W\mapsto\F$ from parameter space $\W$ to a space $\F=\{\Op(\*w):\X\mapsto\Y\vert\*w\in\W\}$ of {\bf parameterized functions} from input space $\X$ to output space $\Y$.
	A {\bf search space} is a set of operations with the same $\W$, $\X$, and $\Y$.
\end{Def}

For example, if $\X=\Y=\R^n$ and $\W=\R^{n\times n}$ then each $\*W\in\W$ defines a parameterized linear layer that for each $\*x\in\X$ returns $\Lin(\*W)(\*x)=\*W\*x$.
Here $\Lin$ is the parameterizable operation and for each $\*W$ the linear map $\Lin(\*W)$ is the parameterized function.

From Definition~\ref{def:po}, we say a search space can {\em express} a specific operation if it contains it.
Crucially, the ability of a parameterizable operation $\Op_1$ to express a parameterized function $\Op_2(\*w)$ output from another operation $\Op_2$ given the right set of weights $\*w$ does {\em not} imply that a search space containing $\Op_1$ can express $\Op_2$.
For example, $\Lin(\*I_n)=\Id(\*W)~\forall~\*W\in\R^{n\times n}$ but $\Lin(\*W)\ne\Id(\*W)~\forall~\*W\ne\*I_n$, so a search space containing the linear operation $\Lin$ cannot express the skip-connection $\Id$, despite the fact that $\Lin$ can be parameterized to compute the identity.

\paragraph{Formalizing Multi-Channel XD-Operations}%\label{subsec:xdmc}
Recall the single-channel XD-operation $\XD_\alpha^\1$ in Equation~\ref{eq:xd1} specified by three-matrix architecture parameter $\alpha=(\*K,\*L,\*M)$.
For input dimension $N\ge1$, every matrix $\*B\in\alpha$ is a Kronecker product of $N$ $K$-matrices of depth $\*d\in\Z_+^3$, i.e. $\*B=\bigotimes_{i=1}^N\*B_i$ for K-matrices $\*B_i\in\C^{n\times n}$ of depth $\*d_{[1]}$, $\*d_{[2]}$, or $\*d_{[3]}$ for $\*B=\*K$, $\*L$, or $\*M$, respectively.\footnote{A depth-$d$ K-matrix is a product of $d$ depth-1 K-matrices.}
Roughly speaking, $\XD_\alpha^\1$ can return any linear operation that is diagonalized by K-matrices and is thus efficient to compute and represent, e.g. any convolution (recall we recover the diagonalization of $\Conv(\*w)$ in Equation~\ref{eq:fourier} by setting $\*K$, $\*L$, and $\*M$ appropriately in Equation~\ref{eq:xd1}).
However, $\XD_\alpha^\1$ cannot represent efficient {\em parameter-free} operations such as skip-connections and average-pooling, both common in NAS.
In particular, the only way to always ignore the model weights $\*w$ is to set one of the K-matrices to zero, producing the zero-operation.
We avoid this by adding a bias $\*b\in\C^{n^N}$ as an architecture parameter, yielding the {\em biased} single-channel XD-operation:\footnote{Zero-padding $\*x$ as well lets the input to be smaller than the output if needed, e.g. for transposed convolutions.}
\begin{equation}\label{eq:biased}
\XD_{\alpha,\*b}^\1(\*w)(\*x)=\Real\left(\*K\diag(\*L\underline{\*w}+\*b)\*M\underline{\*x}\right)
\end{equation}
This lets us define skip-connections (set $\*K=\*M=\*I_{n^N}$, $\*L=\0_{n^N\times n^N}$, and $\*b=\1_{n^N}$) and average-pooling (set $\*K=\*F^{-1}$, $\*L=\0_{n^N\times n^N}$, $\*M=\*F$, and $\*b$ to be $\*F$ multiplied by a pooling filter).

Lastly, we use $\XD_{\alpha,\*b}^\1$ to construct multi-channel ``layers'' that pass multiple input features through multiple channels and re-combine them as multiple output features.
This follows the primary way of using convolutions in deep nets.
The key insight here is that we will share the same parameterizable operation (specified by $\alpha$ and $\*b$) across all channels, just as in convolutional layers.
\begin{Def}\label{def:xd}
	Let $a=(\alpha,\*b,\*C)$ be an architecture parameter containing a triple $\alpha=(\*K,\*L,\*M)$ of Kronecker products of $N$ K-matrices with depths $\*d\in\Z_+^3$, a bias $\*b\in\C^{n^N}$, and channel gates $\*C\in\C^{c\times c}$.\footnote{For simplicity we formalize the case where all $N$ dimensions have the same input size and there is an identical number $c$ of input and output channels; both are straightforward to extend.}
	Using ``$\bigoplus$'' to denote concatenation, the {\bf XD-operation} $\XD_a$ of depth $\*d$ specified by $a$ is a parameterizable operation on parameter space $\W=\R^{c\times c\times\*k}$ consisting of $c^2$ filters of size $\*k\in[n]^N$ that outputs parameterized functions on $\X=\R^{c\times m^N}$ for $m\le n$ mapping every $\*x\in\X$ to\vspace{-5pt}
	\begin{equation}
	\XD_a(\*w)(\*x)=\bigoplus_{i=1}^c\sum\limits_{j=1}^c\*C_{[i,j]}\XD_{\alpha,\*b}^\1(\*w_{[i,j]})(\*x_{[j]})
%	\begin{pmatrix}
%	\sum\limits_{j=1}^c\*C_{[1,j]}\XD_{\alpha,\*b}^\1(\*w_{[1,j]})(\*x_{[j]})\\
%	\vdots\\
%	\sum\limits_{j=1}^c\*C_{[c,j]}\XD_{\alpha,\*b}^\1(\*w_{[c,j]})(\*x_{[j]})
%	\end{pmatrix}
	\end{equation}
\end{Def}
The last architecture parameter $\*C$ allows interpolation between all-to-all layers ($\*C=\1_{c\times c}$), e.g. multi-channel convolutions, and layers where each channel is connected to one other channel ($\*C=\*I_c$), e.g. skip-connections and average-pooling.
We note that we use $\Search_\XD$ to describe the set of operations covered by Definition~\ref{def:xd} and conclude our construction by discussing two properties:
%In particular, note that all XD-operations satisfy the properties in Definitions~\ref{def:eff},~\ref{def:dl}, and~\ref{def:lin}.
\begin{itemize}[leftmargin=*,topsep=-1pt,noitemsep]\setlength\itemsep{2pt}
	\item{\bf Kernel size:} the weight-space available to an XD-operation is $\R^{c\times c\times n^N}$; however, since we will initialize search with existing CNNs, we will zero-pad to have the same weight-space $\R^{c\times c\times k^N}$ as the convolutions with filter size $k\le n$ that they replace.
	This preserves the weight count but also means that if the backbone has $3\times3$ filters our search space will {\em not} contain $5\times5$ convolutions.
	Experimentally, we find that relaxing the constraint to allow this does not significantly affect results on image tasks, so we do not do so in subsequent applications to avoid increasing the weight count.
	\item{\bf Depth:} an XD-operation's depth is a triple describing the depths of its K-matrices $\*K$, $\*L$, and $\*M$.
	Increasing it trades off efficiency for expressivity;
	for example, in the next section we describe operations that we can show are contained in $\Search_\XD$ if $\*L$ or $\*M$ have depth $>1$.
	By default we will set the depth to be the minimum needed to initialize search with the backbone operation.
\end{itemize}

\paragraph{Expressivity of XD-Operations}%\label{subsec:express}
For many papers that replace deep net layers with efficient linear transforms \cite{moczulski2015acdc,dao2020kaleidoscope}, the question of expressivity comes down to the transform capacity.
For example, layers with a K-matrix in every channel can represent a different transform in each, thus allowing the output to be any combination of efficient linear operations.
Our case is less straightforward since we care about expressivity of the search space, not of parameterized functions, and our approach is less-expressive {\em by design} as all channels share K-matrices $\*K$, $\*L$, and $\*M$.
The latter can be thought of as a useful inductive bias on NAS:
the set of XD-operations is still much broader than the set of convolutions, but the way in which model weights are applied is the same across all channels.

Expressivity results are a way to see if this bias is useful or constraining.
Here we summarize some important operations that are 1d XD-operations; 
proofs can be found in the appendix and are straightforward to extend to multi-dimensional inputs.
Formally, there exists $\*d\in\Z_+^3$ such that the set of XD-operations of depth $\*d$ over weights $\W=\R^{c\times c\times k}$ and inputs $\X=\R^m$ for $m\le n$ contains
\begin{enumerate}[leftmargin=*,topsep=-1pt,noitemsep]\setlength\itemsep{2pt}
	\item convolutions with filter size $\le k$, dilation $\le\lfloor\frac{n-1}{k-1}\rfloor$, stride $\le n-1$, and arbitrary channel groups.
	\item parameter-free operations $\Id$, $\Zero$, and $\AvgP_s$ for any kernel size $s\le n$.
	\item composing 1 or 2 with multiplication of all input or output channels by a bounded-depth K-matrix.
\end{enumerate}
Note this does not account for {\em all} important XD-operations, e.g. we show in the appendix that they also express Fourier Neural Operators \citep{li2021fno} with $\le\lfloor k/2\rfloor$ modes and any transposed convolutions whose stride equals the dilated kernel size.\footnote{This restriction still includes transposed convolutions used in well-known architectures such as U-Net \citep{ronneberger2015unet}.}
Still, the first two items account for non-separable variants of most operations considered in past NAS work in computer vision, excluding the nonlinear $\MaxP$ \citep{ying2019nasbench101,dong2020nasbench201}.
Note depthwise-separable convolutions {\em are} contained in the set of compositions of XD-operations.
The third item implies that XD-operations can express the basic and diffusion graph convolutions over fixed graphs \citep{kipf2017gcn,li2018dcrnn}:
both are point-wise convolutions composed with sparse multiplication by a modified adjacency matrix, which K-matrices can represent efficiently.

%Apart from understanding our search space, a chief motivation for these results is to enable initializing search using the operations of existing backbones.
As a concrete example, consider dilated convolutions, which for $k>1$ and dilation factor $d\ge1$ apply filters of effective size $(k-1)d+1$ with nonzero entries separated by $d-1$ zeros.
One could hope to express the application of $\DilC_{k,d}$ to an input $\*x\in\R^n$ in the single-channel setting as $\*F^{-1}\diag(\*F\diag(\*p_{k,d})\underline{\*w})\*F\*x$, where $\*p_{k,d}\in\{0,1\}^n$ zeroes out appropriate entries of $\underline{\*w}$, but this requires filter size $(k-1)d+1>k$, increasing the number of weights.
Instead, we can use a permutation $\*P_{k,d}\in\{0,1\}^{n\times n}$ before the DFT to place the $k$ entries of $\underline{\*w}$ into dilated positions:
\begin{equation}\label{eq:dilated}
	\DilC_{k,d}(\*w)(\*x)=\*F^{-1}\diag(\*F\*P_{k,d}\underline{\*w})\*F\*x
\end{equation}
As permutations are depth-2 K-matrices \citep{dao2020kaleidoscope}, we can express $\DilC_{k,d}$ with an XD-operation of depth $(1,3,1)$, with $\*K=\*F^{-1}$, $\*L=\*F\*P_{k,d}$, and $\*M=\*F$.

%% file: chrysalis.tex
% !TEX root = main.tex

\vspace{-1mm}
\section{Finding and Evaluating XD-Operations}\label{sec:chrysalis}
\vspace{-1mm}

This section outlines a simple procedure that we use to evaluate XD-operations.
Recall that NAS methods specify architectures by assigning operations to each edge $(u,v,\Op)$ of a computational graph.
We aim to simultaneously find good operations and model weights, a goal distinct from the classic {\em two-stage} NAS formulation, which finds assignments in an initial search phase before training the resulting architecture from scratch \citep{ying2019nasbench101}.
However, the use of weight-sharing \citep{pham2018enas} extends NAS to {\em one-shot} objectives where weights and architectures are jointly optimized.
Under weight-sharing, architecture parameters become weights in a larger ``supernet,'' extending the hypothesis class \citep{li2021gaea}.

To assess XD-operations directly we assume the user provides a starter network with existing edge labels $\Op_{u,v}$ as a backbone.
We transform this into a weight-sharing supernet by reparameterizing each operation $\Op_{u,v}$ as an XD-operation $\XD_{a_{u,v}}$ with architecture parameter $a_{u,v}$.
Then we simultaneously train both $a_{u,v}$ and the model weights $\*w_{u,v}$ associated with each edge as follows:
\begin{itemize}[leftmargin=*,topsep=-1pt,noitemsep]\setlength\itemsep{2pt}
	\item {\bf Architecture parameters} $a_{u,v}$ are initialized using the original operation used by the CNN backbone by setting $\Op_{u,v}=\XD_{a_{u,v}}$;
	$a_{u,v}$ is then updated via SGD or Adam \citep{kingma2015adam}.
	We tune step-size, momentum, and the number of ``warmup'' epochs: 
	initial epochs during which only model weights $\*w_{u,v}$ are updated.
	This can be viewed as a specialized step-size schedule.
	\item {\bf Model weights} $\*w_{u,v}$ are initialized and updated using the routine provided with the backbone.
\end{itemize}

%\section{Application: Image Classification}

\vspace{10mm}
This approach allows us to use established topologies and optimizers while searching for new operations, thus aligning with the goal for Sections~\ref{sec:pde}, \ref{sec:protein}, and~\ref{sec:seq}: 
to improve upon the CNN backbones that practitioners often use as a first attempt.
As a simple example, we start by applying the procedure to image classification.
Since this is not the main objective of our work, we treat it as a warmup and consider two datasets:
CIFAR-10 and a variant where the images' rows and columns are permuted.
On CIFAR-10 we do {\em not} expect to see much improvement from XD-operations over the CNN backbone used to initialize search, as convolutions are already the ``right'' operation for images.
On the other hand,  the ``right'' operation on permuted data, at least in layer one, is an inverse permutation followed by convolution;
as this is an XD-operation\footnote{Recall $\Search_\XD$ includes compositions of convolutions with multiplication by a K-matrix, e.g. a permutation.}, here we do hope to see improvement.

Using LeNet \citep{lecun1999lenet} and ResNet-20 \citep{he2016resnet} as backbones, we compare applying our algorithm to XD-operations with two baselines:
(1) using just the backbone CNN and (2) applying a similar method to the relaxed set $\tilde\Search_\disc$ of DARTS operations from Section~\ref{sec:relax}.
To optimize over $\tilde\Search_\disc$ we take an approach similar to DARTS:
parameterize the simplex using a softmax and apply Adam.
We experiment with both a uniform initialization and one biased towards the backbone's operation.
While both $\Search_\XD$ and $\Search_\disc$ contain LeNet's $\Conv_{5\times 5}$ and ResNet's $\Conv_{3\times 3}$ and $\Id$, for LeNet's $\MaxP_{3\times 3}$ layer we initialize with the closest operation.
For direct comparison, both search spaces employ weights with maximum filter size $5\times 5$ and for both we evaluate the shared weights rather than retraining, which we find hurts $\tilde\Search_\disc$.
We set the XD-operations' depth to $\*d=\*3_3$ to express the dilated convolutions in $\Search_\disc$ and convolutions composed with permutations.

In Table~\ref{tab:cifar}, we see that while both the relaxed discrete NAS operations and XD-operations perform comparably on regular images, XD-operations achieve around 15\% better accuracy with both backbones when the images are permuted.\footnote{Full accuracy can be recovered via an auxiliary loss encouraging permutation-like K-matrices \citep{dao2020kaleidoscope}.}
Note that even networks obtained by running state-of-the-art NAS procedures such as GAEA PC-DARTS \citep{li2021gaea} and DenseNAS \citep{fang2020densenas} on permuted CIFAR-10 achieve only 66.3\% and 61.6\% accuracy, respectively, despite using millions more parameters than ResNet-20.
While it is not straightforward to understand the recovered XD-operations that perform so well, we can use the relative Euclidean distance of their architecture parameters from initialization as a proxy for novelty;
in Figure~\ref{fig:div} we see that on regular images our procedure finds operations that are quite similar to convolutions, but on permuted data they are much further away.
%There they even significantly outperform standard cell-based search, despite the latter's much more expensive and intensive training routine (600 epochs v. the 200 used for ResNet-20), much larger model size (2.3 million parameters v. 1.2 million), and higher performance on regular images.
These results show that to enable NAS on diverse data, we will need a search space that contains truly novel operations, not just combinations of existing ones.
In the remainder of the paper, we study more diverse and realistic tasks that show further evidence that $\Search_\XD$ is a strong candidate for this.

\begin{figure}[!t]
	\begin{minipage}{0.6\linewidth}
		\begin{threeparttable}[H]
			\captionof{table}{\label{tab:cifar}
				Search space comparison on CIFAR-10.
				Validation accuracies are averages of three trials.
				While we use small CNNs for exploration, XD-operations can also be used with high-performance backbones to obtain $>95\%$ accuracy (c.f. the appendix).
				\vspace{-4pt}
			}
			\begin{tabular}{lccc}
				\toprule
				{\bf Backbone} & & Permuted & Cost \\
				~~search space & CIFAR-10 & CIFAR-10$^\ast$ & (hours$^\dagger$) \\
				\midrule
				{\bf LeNet} & $75.5\pm0.1$ & $43.7\pm0.5$ & 0.3 \\
				~~$\tilde\Search_\disc$ & $75.6\pm3.4$ & $47.7\pm1.0$ & 1.0 \\
				~~$\Search_\XD$ & $77.7\pm0.7$ & $63.0\pm1.0$ & 0.9 \\
				\midrule
				{\bf ResNet-20} & $91.7\pm0.2$ & $58.6\pm0.7$ & 0.6 \\
				~~$\tilde\Search_\disc$ & $92.7\pm0.2$ & $58.0\pm1.0$ & 5.3 \\
				~~$\Search_\XD$ & $92.4\pm0.2$ & $73.5\pm1.6$ & 5.6 \\
				\bottomrule
				%			DARTS Cell$^\ddagger$ & $96.0\pm0.2$ & $66.3\pm0.5$ & 28.6 \\
				%			\hline
			\end{tabular}
			\begin{tablenotes}\footnotesize
				\item[$\ast$] No data augmentation used in the permuted case.
				%			\item[$\dagger$] On a V100 GPU; time for DARTS Cell is training cost only.
				%			\item[$\ddagger$] Search using GAEA PC-DARTS \citep{li2021gaea}; training using ``base" routine from \citet{yang2020nas}.\vspace{-4pt}
			\end{tablenotes}
		\end{threeparttable}
	\end{minipage}
	\hfill
	\begin{minipage}{0.4\linewidth}
		\centering
		\includegraphics[width=\linewidth]{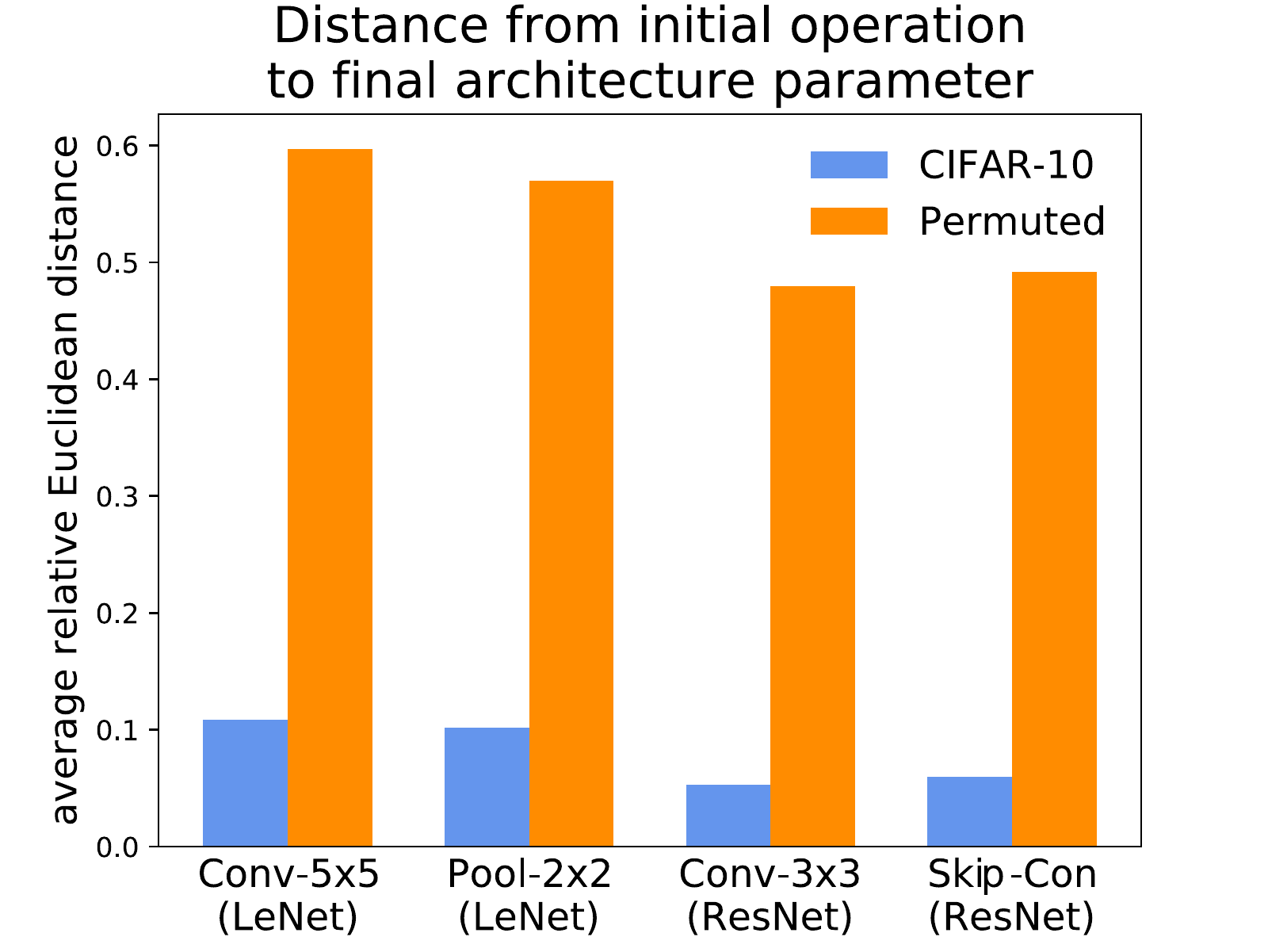}
		\vspace{-3pt}
		\caption{\label{fig:div}
			On permuted images, where convolutions are not the ``right'' operation, we find XD-operations that are farther away from the operations of the initial CNN backbone.
		}
		
	\end{minipage}
\end{figure}

%% file: pde.tex
% !TEX root = main.tex

\begin{figure}[!t]
	\begin{minipage}{\linewidth}
		\centering
		\hspace{8mm}
		\includegraphics[width=0.4\linewidth]{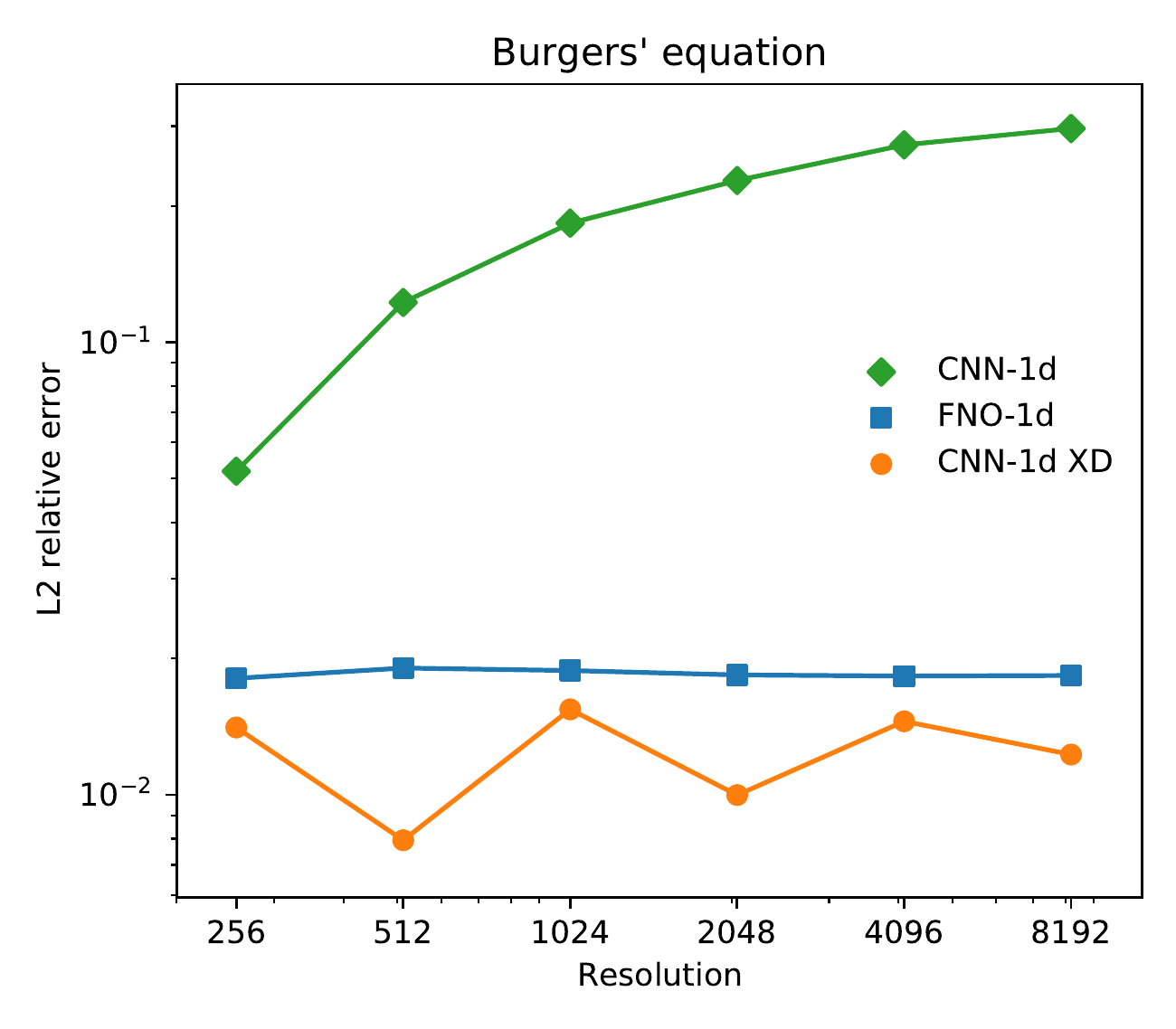}
		\hfill
		\includegraphics[width=0.4\linewidth]{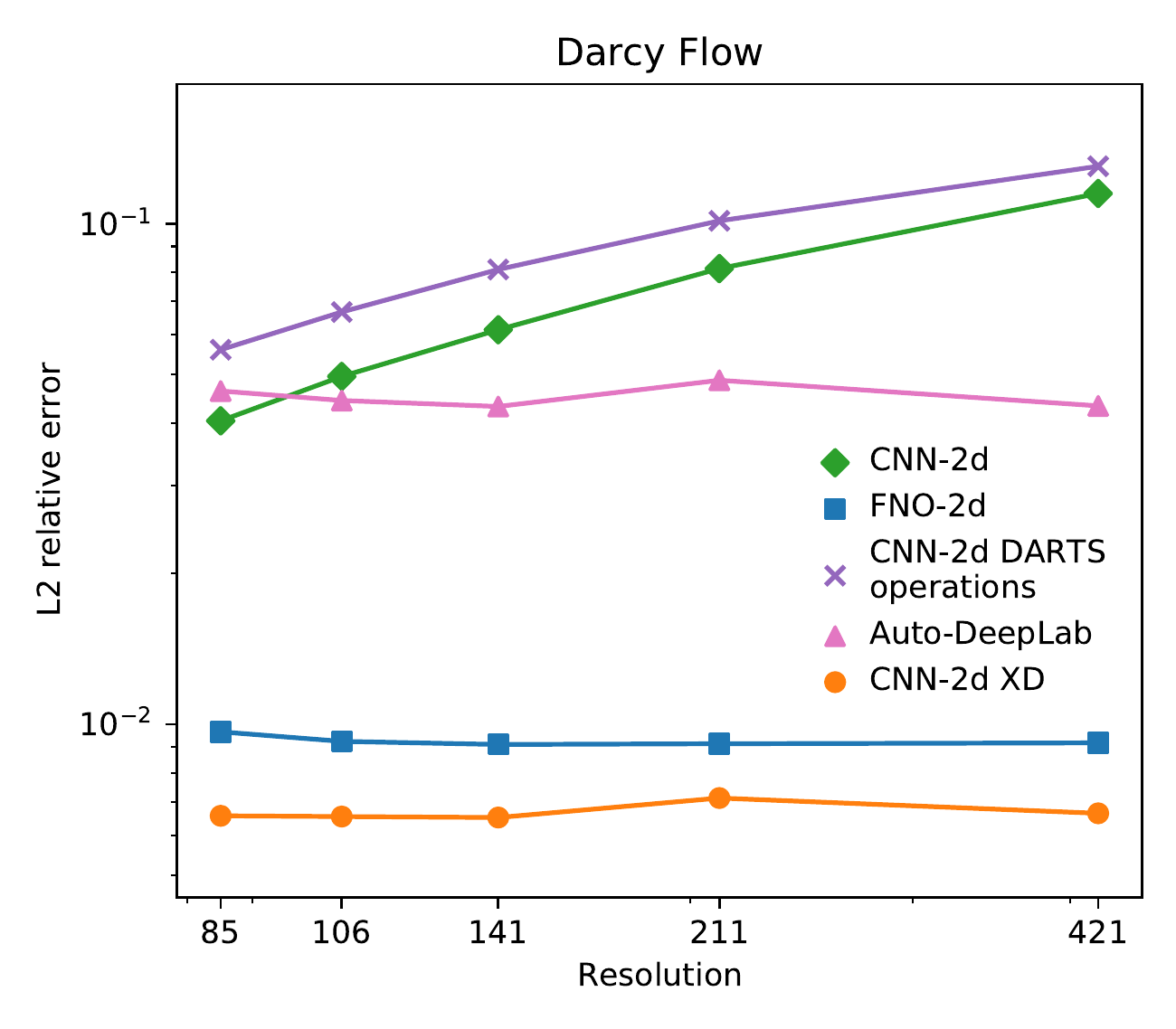}
		\hspace{8mm}
		\vspace{-9pt}
		\caption{
			Relative error on Burgers' equation (left) and Darcy Flow (right) across different resolutions.}
		\vspace{-8pt}
		\label{fig:pde}
	\end{minipage}
	%	\hfill
	%	\begin{minipage}{0.44\linewidth}
	%		\centering
	%		\begin{threeparttable}
	%			\captionof{table}{\label{tab:navierstokes}
	%				Relative test error on the 2d Navier-Stokes equations at  different settings of the viscosity $\nu$ and time steps $T$. 
	%				Best results in each setting are {\bf bolded}.
	%			}
	%			\footnotesize
	%			\begin{tabular}{lcc}
	%				\hline
	%				& $\nu=10^{-4}$ & $\nu=10^{-5}$ \\
	%				Method                     & $T=30$ & $T=20$  \\ 
	%				\hline
	%				%FNO-3D \citep{li2021fno}                &    0.1918    &    0.1893   \\
	%				%FNO-2D \citep{li2021fno}                &     {\bf 0.1559}   &   {\bf 0.1556}    \\
	%				%U-Net \citep{li2021fno}                &     0.2051   &    0.1982   \\
	%				%TF-Net \citep{li2021fno}                &     0.2253   &   0.2268    \\
	%				%ResNet \citep{li2021fno}                &    0.2871    &   0.2753    \\
	%				%\hline
	%				CNN-3d                &    0.325    &  0.278    \\
	%				FNO-3d (repr.)                           &   0.182     &  0.177   \\
	%				{\bf CNN-3d XD}  &     {\bf 0.172}  &  {\bf 0.168}     \\
	%				\hline
	%			\end{tabular}
	%			%		\begin{tablenotes}\footnotesize
	%			%			\item[$\ast$] No data augmentation used in the permuted case.
	%			%			\item[$\dagger$] Training using ``base" routine from \citet{yang2020nas}.
	%			%		\end{tablenotes}
	%		\end{threeparttable}
	%	\end{minipage}
\end{figure}

\vspace{-1mm}
\vspace{10mm}
\section{Application: Learning to Solve Partial Differential Equations}\label{sec:pde}
\vspace{-1mm}

As our first non-vision application, we consider the task of solving PDEs, an important application area of ML in the natural sciences \citep{li2015butterfly, li2018multidimensional, sirignano2018dgm}.
In our setup, data generated by classical PDE solvers is used to learn functions from some initial condition or setting to the corresponding PDE solution, with the goal of replacing the solver by a deep net forward pass;
the latter can be orders of magnitude faster.
A recent state-of-the-art approach for this introduces Fourier Neural Operators \citep{li2021fno}, operations that significantly improve upon previous neural approaches across three different PDE settings.
To evaluate the ability of XD-operations to compete with such custom-designed operations starting from simple CNN backbones, we will investigate the same three PDEs that they study: Burgers' equation, Darcy Flow, and the 2d Navier-Stokes equations, which involve 1d, 2d, and 3d data, respectively.
The first two are studied across multiple resolutions, while the last one is studied at different viscosities.

%As a setup, we consider the same setting as that of \citet{li2021fno}, specifically the Burgers' equation, Darcy Flow, and the Navier-Stokes equations;
%as the data domain of these tasks is 1d, 2d, and 3d, respectively, this application also allows us to examine the performance of XD-operations across several dimensionalities. We follow the setup from \citet{li2021fno} and evaluate on the same 1d Burgers', 2d Darcy Flow, and 2d Navier-Stokes datasets used in their experiments. Per \citet{li2021fno}, the 2d Navier-Stokes problem involves 3d data which treats time as a third dimension. 

As before, we start with a simple CNN backbone---the type a scientist might use in a first attempt at a solution---and replace all convolutions by XD-operations.
We initially hope to do better than this backbone, but ambitiously also hope to compete with the custom-designed FNO.
The specific CNN we use is simply the FNO architecture of the appropriate dimension $N$ but with all $N$-dimensional FNOs replaced by $N$-dimensional convolutions;
this performs similarly to their CNN baselines \cite{li2021fno}.
In all cases we compare mainly to the CNN backbone and our reproduction of the FNO results, as the latter exceeds all other neural methods;
a complete results table is provided in the appendix.
Our reproduction of FNO is slightly worse than their reported numbers for Burgers' equation and slightly better in the other two settings.
Note that on the Navier-Stokes equations we only compare to the 3d FNO on the two settings in which we were able to reproduce their approach;
moreover, we do {\em not} compare to their use of a 2d FNO plus a recurrent net in time, but in-principle XD-operations can also be substituted there.
In the 2d Darcy Flow case we also include comparisons to DARTS operations in the simple CNN backbone, as in Section~\ref{sec:chrysalis}, and to Auto-DeepLab (AutoDL) \citep{liu2019autodl}, a well-known NAS method for dense prediction.
For evaluating XD-operations we again follow the procedure in Section~\ref{sec:chrysalis}, in which we tune only the architecture optimizer;
notably, we do this only at the lowest resolutions.
At all dimensions we use XD-operations of depth $\*d=\1_3$;
in addition, in dimensions $N>1$ we fix the architecture biases $\*b$ and channel gates $\*C$ to $\0$ and $\1$, respectively, to conserve memory at higher resolutions.
At lower ones we find that the performance difference is negligible.

We report our results for the Burger's equation and Darcy Flow in Figure~\ref{fig:pde};
for 2d Navier-Stokes the results are in Table~\ref{tab:navierstokes}.
In all cases we dramatically outperform the CNN backbone used to initialize XD-operations;
furthermore, we also achieve better error than FNO, despite it being custom-made for this problem.
In particular, we find that XD-operations have higher {\em training error} but generalize better (c.f. the appendix).
Figure~\ref{fig:pde} also shows that XD-operations perform consistently well across resolutions, a major advantage of FNOs over previous methods, whose performance was tightly coupled to the discretization \citep{li2021fno}.
Notably, CNN performance worsens with higher resolution, unlike that of XD and FNO.
Finally, we also substantially outperform DARTS operations and AutoDL in 2d, although the latter is at least consistent across resolutions.
These results provide strong evidence that XD-operations are a useful search space for discovering neural operations, even in domains where the convolutions used to initialize them perform much worse than state-of-the-art.
Note that these results do come at a cost of slower training and inference:
XD-operations are roughly an order of magnitude slower than FNOs, despite having fewer parameters in 2d and 3d.
This still yields solvers one-to-two orders of magnitude faster than classical solvers, maintaining usefulness for the problem.

\begin{figure*}[!t]
	\begin{minipage}{\linewidth}
		\centering
		\begin{threeparttable}
			\captionof{table}{\label{tab:navierstokes}
				Relative test error on the 2d Navier-Stokes equations at  different settings of the viscosity $\nu$ and time steps $T$. 
				Best results in each setting are {\bf bolded}.
			}
			\begin{tabular}{lcccccc}
				\toprule
				&&& $\nu=10^{-4}$, $T=30$ &&& $\nu=10^{-5}$, $T=20$ \\
				\midrule
				%FNO-3D \citep{li2021fno}                &    0.1918    &    0.1893   \\
				%FNO-2D \citep{li2021fno}                &     {\bf 0.1559}   &   {\bf 0.1556}    \\
				%U-Net \citep{li2021fno}                &     0.2051   &    0.1982   \\
				%TF-Net \citep{li2021fno}                &     0.2253   &   0.2268    \\
				%ResNet \citep{li2021fno}                &    0.2871    &   0.2753    \\
				%\hline
				CNN-3d (our baseline)              &&&    0.325    &&&  0.278    \\
				FNO-3d (reproduced)                           &&&   0.182     &&&  0.177   \\
				{\bf CNN-3d XD}  (ours) &&&     {\bf 0.172}  &&&  {\bf 0.168}     \\
				\bottomrule
			\end{tabular}
			%		\begin{tablenotes}\footnotesize
			%			\item[$\ast$] No data augmentation used in the permuted case.
			%			\item[$\dagger$] Training using ``base" routine from \citet{yang2020nas}.
			%		\end{tablenotes}
		\end{threeparttable}
	\end{minipage}
\end{figure*}

%% file: protein.tex
% !TEX root = main.tex

\begin{figure*}[!t]
	\begin{minipage}{\linewidth}
		\centering
		\includegraphics[width=\columnwidth]{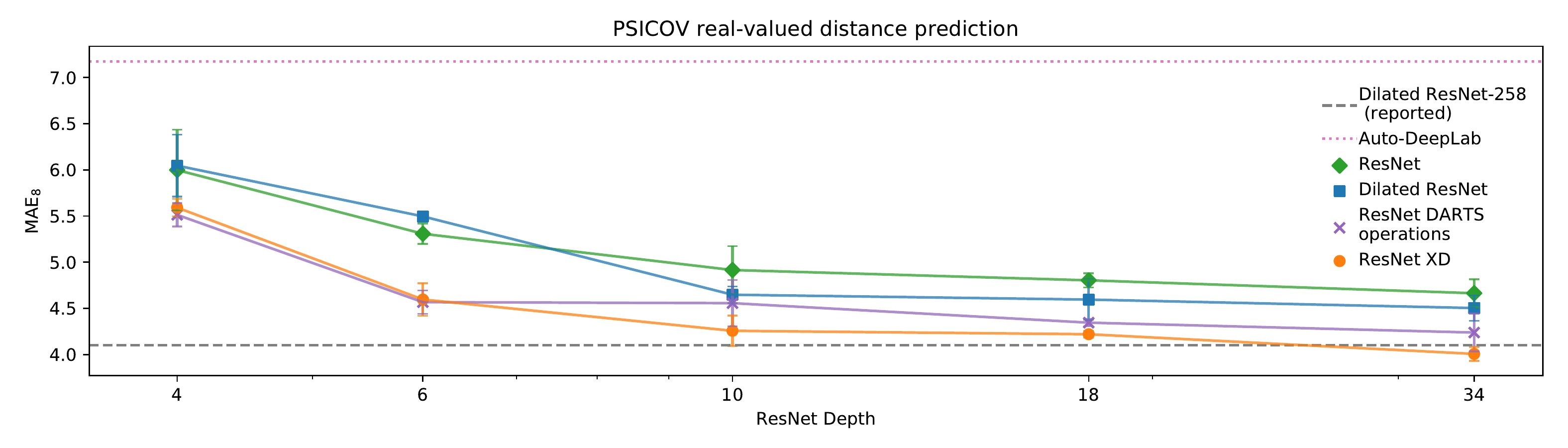}
		\vspace{-9pt}
		\caption{
			ResNet XD outperforms both baseline and dilated ResNets on PSICOV.
			At the highest depth we test we also outperform the reported MAE$_8$ of the much deeper Dilated ResNet-258~\citep{adhikari2020}.
		}
		\label{fig:protein}
	\end{minipage}
\end{figure*}

\vspace{-1mm}
\section{Application: Real-Valued Distance Prediction for Protein Folding}\label{sec:protein}
\vspace{-1mm}

As a second scientific application, we consider the task of inferring the 3d ``folded'' structure of a polypeptide chain, which yields important insights into the function of the resulting protein~\cite{psicov}. This problem is a high-priority challenge in biology and has recently seen significant ML-driven advances from deep learning methods such as AlphaFold~\cite{alphafold, alphafold2} and PDNET~\cite{adhikari2020}.
These typically involve training a network to predict pairwise physical distances between residues in the chain.
% e.g., major performance gains achieved by AlphaFold 2 on the CASP-14 assessment \cite{alphafold2}.
We work with the PDNET benchmark, which consists of a training set of 3,356 proteins, a validation set of 100 of proteins, and the PSICOV \cite{psicov} test set of 150 proteins. 
PDNET is designed to be more accessible than datasets used by large-scale methods such as AlphaFold, which are not always publicly available and/or require massive compute \cite{alphafold, alphafold2}.
%Each example protein consists of 57 pairwise features of elements in the polypeptide chain represented as a $57 \times S \times S$ input tensor where $S$ is the length of the sequence. Targets comprise an $S \times S$ pairwise distance matrix measured in angstroms ($\angstrom$). 
%Following the procedure from \cite{}, we train and validate on $57 \times 128 \times 128$ random patches from each protein. 
%At test time, we predict using the full, un-cropped $57 \times S \times S$ input. Since XD does not support arbitrary input dimensions, we predict on disjoint $57 \times 128 \times 128$ patches of the input using multiple forward passes and concatenate the outputs to form the full $S \times S$ predicted distance matrix. 
%For all methods, we take the average of the upper and lower triangle of the predicted distance matrix and employ the evaluation procedure described in \cite{}, which proposed the MAE$_8$ metric. MAE$_8$ computes the mean absolute error (MAE) across indices where the ground truth distance is less than 8 $\angstrom$. Additionally, following \cite{}, we only consider pairs separated by at least 23 elements in the sequence. Intuitively, this measures MAE for physically close yet long-range pairs, which is more informative of global protein shape than standard MAE. 
We follow the PDNET training procedure~\cite{adhikari2020} and evaluate test set performance using their MAE$_8$ metric for assessing long-range distances. 

%For this task, our goal is to outperform the pre-activation ResNet-like architecture of \cite{}, which was specifically tailored to this problem, and is characterized by its use of a cyclically increasing dilation rates across ResNet blocks (namely, 1, 2, and 4)--we will refer to this baseline as Dilated ResNet. This architecture tends to slightly outperform a standard pre-activation ResNet adapted to this task--which we'll simply refer to as ResNet. 
As before we start with simple CNN backbones---in this case ResNets.
We choose this to compare most directly to the custom-designed architecture used by PDNET, consisting of a Dilated ResNet characterized by its use of a cyclically increasing dilation rate across ResNet blocks \cite{adhikari2020}.
%As a baseline, we consider the dilated pre-activation ResNet of \cite{adhikari2020}, which was specifically tailored to this problem, and is characterized by its use of a cyclically increasing dilation rates across ResNet blocks--we will refer to this baseline as Dilated ResNet. 
At a sufficient depth, the Dilated ResNet is shown to outperform a standard pre-activation ResNet adapted to this task~\cite{adhikari2020}. Our goal will be to see whether we can start with the vanilla ResNet and use XD to outperform both it and the specialized Dilated ResNet. 
We also aim to outperform the DARTS operations baseline from the previous two sections as well as the AutoDL NAS approach for dense prediction.
%While we can indeed initialize our XD-operations as the specialized dilated convolutions, our approach will be to instead initialize our XD-operations as standard convolutions without dilations and use ResNet as a backbone, with the goal of showing that a standard ResNet with standard convolution-initialized XD-operations outperforms the specialized Dilated ResNet. 
%While we can indeed use the specialized Dilated ResNet as our XD backbone, our approach will be to instead initialize our XD-operations as standard convolutions without dilations and use ResNet as the backbone. 
%The goal, then, is to show that ResNet with standard convolution-initialized XD-operations outperforms the specialized Dilated ResNet. 
We use XD-operations of depth $\*d=\1_3$ and fix the architecture biases and channel gates as before to conserve memory.
We evaluate architectures of different depths---4, 6, 10, 18, and 34---by varying the number of ResNet blocks used in the backbone architecture and baseline. 
%In particular, we consider depths 4, 6, 10, 18, and 34 (1, 2, 4, 8, and 16 blocks, respectively). 

We report the results as averages across three trials for each depth in Figure~\ref{fig:protein}. 
Notably, while Dilated ResNet slightly outperforms ResNet, ResNet XD outperforms both dilated and standard ResNets at all depths. This provides further evidence that XD-operations can outperform specialized operations for diverse domains, even when initialized naively as standard convolutions. 
XD also outperforms AutoDL, which does poorly, and DARTS operations, except at the two smaller depths where performance is similar.
Moreover, our ResNet-34 XD's MAE$_8$ of 4.0 also improves upon PDNET's reported MAE$_8$ of 4.1 attained by the much deeper Dilated ResNet-258~\cite{adhikari2020};
however, in our reproduction Dilated ResNet-258 achieved an MAE$_8$ of $3.5$.
Given the trend in Figure~\ref{fig:protein}, where XD-operations consistently improve the backbone architecture of the same depth, we conjecture that ResNet-258 XD could further improve upon this result.
We leave scaling XD-operations to such deeper networks to future work. 

\begin{figure*}[!t]
%	\begin{minipage}{0.46\linewidth}
%		\centering
%		\includegraphics[width=1.0\columnwidth]{figures/protein.pdf}
%		\vspace{-18pt}
%		\caption{
%			ResNet XD outperforms both baseline and dilated ResNets on  PSICOV, as well as the reported MAE$_8$ of Dilated ResNet-258.
%		}
%		\label{fig:protein}
%	\end{minipage}
%	\hfill
	\begin{minipage}{\linewidth}
		\centering
		\begin{threeparttable}
			\captionof{table}{\label{tab:seq}
				XD-operations compared to recent results in music modeling.
				We report average loss across three trials.
				The best result on each task is {\bf bolded}.
			}
			\begin{tabular}{lcccccccc}
				\toprule
				Method (source) &&&& JSB Chorales &&&& Nottingham \\
				\midrule
				Best recurrent \citep{bai2018tcn} &&&& 8.43 &&&& 3.29  \\
				TCN \citep{bai2018tcn} &&&& 8.10 &&&& 3.07 \\
				Transformer \citep{wang2020rtransformer} &&&& - &&&& 3.34 \\
				R-Transformer \citep{wang2020rtransformer} &&&& - &&&& {\bf 2.37} \\
				\midrule
				Undilated TCN (our baseline) &&&& $8.16\pm0.04$ &&&& $3.23\pm0.02$ \\
				TCN (reproduced) &&&& $8.17\pm0.01$ &&&& $2.97\pm0.01$ \\
				{\bf Undilated TCN XD} (ours) &&&& ${\bf 8.07\pm0.01}$ &&&& $2.84\pm0.02$ \\
				%			{\bf TCN backbone XD} (ours) & ${\bf 8.07\pm0.02}$ & $2.81\pm0.05$ & \\
				\bottomrule
			\end{tabular}
			%		\begin{tablenotes}\footnotesize
			%			\item[$\ast$] We use depth $(1,3,1)$ to search for dilated convolutions.\vspace{-4pt}
			%		\end{tablenotes}
		\end{threeparttable}
	\end{minipage}
\end{figure*}

%% file: seq.tex
% !TEX root = main.tex

\vspace{-1mm}
\section{Application: Music Modeling}\label{sec:seq}
\vspace{-1mm}

%\begin{table*}[!t]
%	\centering
%	\begin{threeparttable}
%		\begin{tabular}{lcccc}
%			\hline
%			& Permuted MNIST$^\ast$ & ~JSB Chorales~~ & ~~~~Nottingham~~~~ & ~~~Penn Treebank \\
%			Method (source)  & (error) & (loss) & (loss) & (perplexity) \\
%			\hline
%			LSTM \citep{bai2018tcn} & 14.3 & 8.45 & 3.29 & 78.93 \\
%			GRU \citep{bai2018tcn} & 12.7 & 8.43 & 3.46 & 92.48 \\
%			RNN \citep{bai2018tcn} & 74.7 & 8.91 & 4.05 & 114.50 \\
%			TCN backbone \citep{bai2018tcn} & 2.8 & 8.10 & 3.07 & 88.68 \\
%			TrellisNet \citep{bai2019trellis} & 1.87 & - & - & {\bf 54.19} \\
%			R-Transformer \citep{wang2020rtransformer} & - & - & {\bf 2.37} & 84.38 \\
%			HiPPO-LegS \citep{gu2020hippo} & {\bf 1.7} & - & - & - \\
%			\hline
%			TCN backbone (reproduced) & $2.89\pm0.04$ & $8.17\pm0.01$ & $2.97\pm0.01$ & $88.49\pm0.31$ \\
%			{\bf TCN backbone XD} (ours) & ${\bf 1.75\pm0.11}$ & ${\bf 8.07\pm0.02}$ & $2.81\pm0.05$ & $84.11\pm0.25$ \\
%			\hline
%		\end{tabular}
%		\begin{tablenotes}\footnotesize
%			\item[$\ast$] We use depth $\*d=(3,3,3)$ XD-operations for permutations;
%			elsewhere we use $(1,3,1)$ to warm-start with dilated convolutions.\vspace{-4pt}
%		\end{tablenotes}
%		\caption{\label{tab:seq}
%			XD-operations applied to TCNs compared to recent empirical results in sequence modeling.
%			Our results are averages of three trials.
%			Methods achieving within one deviation of the best performance are {\bf bolded}.
%		}
%	\end{threeparttable}
%\end{table*}

Our final application is to music modeling, i.e. learning to predict the next note from sheet music~\citep{allan2005chorales}.
%This is closely related to language modeling, which is already well-studied in ML and so we do not focus on it.
%While its text application is heavily studied in ML, there are numerous other applications such as in music  and biology \citep{chen2019biological}.
The dominant approaches for such tasks are recurrent nets~\citep{hochreiter1997lstm} and Transformers~\cite{vaswani2017attention}, but recent work has shown that specially-designed convolutional models can also be made competitive at similar model sizes \citep{bai2018tcn,bai2019trellis}.
We will consider the temporal convolutional network (TCN) \cite{bai2018tcn}, which improves upon a regular CNN by having the dilation factor grow exponentially across layers.
%We will use the simple TCN of \citet{bai2018tcn} as a backbone network to examine the potential for XD-operation search to improve performance.
%Our primary goal will be to exceed the TCN baseline across several domains:
%flat permuted images, music, and text.
%We also compare to more recent, human-designed architectures.
The tasks we study are on the JSB Chorales and Nottingham corpora, used in the original evaluation of TCNs \citep{bai2018tcn}.
As the baseline we take the TCN and set all dilation factors to one (undilated);
our goal will be to start with this undilated network and match or outperform the custom dilation design of the TCN.
%As before, we use the same network size and model weight optimization as the backbone, and we initialize XD-operation using its operations, in this case temporal convolutions;

%As discussed in Section~\ref{subsec:express}, to handle dilations we simply need the middle K-matrix $\*L$ to have depth 3.
%Notably, we can painlessly handle dilation size growing exponentially with layer number as in TCNs, while standard NAS operation spaces like DARTS only contain dilations of size 1 and 2.

The results presented in Table~\ref{tab:seq} show that we achieve this goal, as we outperform both the undilated baseline and the TCN on both tasks.
While the simple undilated backbone that we initialize with turns out to already match the TCN on JSB Chorales, on Nottingham our approach demonstrates that XD-operations can be used to outperform hand-designed architectures starting from vanilla CNNs.\footnote{In the appendix we report similar improvements on two other tasks on which TCNs were evaluated---
permuted MNIST and Penn TreeBank---that we do not discuss in detail as our focus is on under-explored tasks.}
Where possible we also compare to other known results;
XD-operations outperforms all of these except the R-Transformer \citep{wang2020rtransformer}, a model combining recurrent nets and self-attention, on Nottingham.

%Notably, our method is competitive with several more sophisticated approaches, exceeding TrellisNet on permuted MNIST---where we match the best-known result, HiPPO-LegS---and improving upon R-Transformer on Penn Treebank (PTB).
%Note that TrellisNet, whose PTB performance exceeds that of the best recurrent NAS cells \citep{bai2019trellis}, also uses convolutions and thus may be improved by XD-operations;
%we do not evaluate this due to its training cost (TrellisNet is 2.5 times larger than the TCN).

Together with our results on PDEs and proteins, our study of music modeling provides further evidence that XD-operations can effectively find good operations using standard backbones on diverse tasks.
One notable difficulty here is causality enforcement:
making sure the input data does not contain the target when predicting the next entry.
While TCNs can efficiently do so via temporal shifts, we do it in a brute-force manner by treating sequences of length $n$ as $n-1$ data-points with masked targets.
This is expensive and thus limits our evaluation to small music tasks.
A fruitful direction for future work is thus to examine whether it is possibly to directly enforce causality in XD-operations, e.g. by forcing architecture parameters $\*K$ and $\*M$ to be lower triangular;
since a product of lower triangular matrices is again lower triangular, the entire operation is then a multiplication of the input sequence by a lower triangular matrix, which suffices to prevent causality violations.
%One alternative approach is to study applications where the full sequence is available, e.g. machine translation \citep{stahlberg2020nmt} or question answering \citep{rajpurkar2016squad}.

%% file: conclusion.tex
% !TEX root = main.tex

\vspace{-1mm}
\section{Conclusion}\label{sec:conc}
\vspace{-1mm}

This work aims to transition NAS from combining existing operations designed for vision and text to finding novel and effective operations in many domains.
To do so we introduced a new search space of XD-operations and demonstrated its effectiveness on diverse tasks.
Combining XD-operations with standard topology-search NAS, warm-starting search from non-standard operations such as graph convolutions and FNOs,\footnote{
	In this direction, we found that initializing XD with FNO did {\em worse} than initializing with convolutions on Burgers' equation and Darcy Flow, a surprising result given how much better FNO is than the baseline CNN. 
	Similarly, initializing XD with convolutions dilated as in the original TCN did not lead to significant improvement, except in one setting, over undilated initialization.
	See the appendix for more details and results.
}
improving the computational limitations described earlier, and constructing spaces containing missing operations such as BatchNorm \citep{ioffe2015batchnorm} and self-attention \citep{vaswani2017attention} are all promising future directions.
Finally, note that our goal---lowering the barrier for applying ML---necessarily comes with the possibility of misuse.
Mitigating this involves developing tools for application-specific concerns, e.g. privacy and fairness, that go beyond the error metrics we target.

%% file: expressivity.tex
% !TEX root = main.tex

\newpage
\section{Expressivity Results}\label{app:expressivity}

Here we collect results on the expressivity of the set of $\XD$-operations.
For simplicity, our results will be in the following single-dimensional ($N=1$) setting:
\begin{Set}\label{app:set:single}
	We consider input spaces of form $\X=\R^{c\times m}$ for input size $m\in\N$ and channel count $c\in\N$ and parameter spaces $\W=\R^{c\times c\times k}$ for filter size $k\in[n]$, where output size $n\ge m$ is a power of 2.
\end{Set}
It is straightforward to extend the results to multiple dimensions using Kronecker products and to input sizes other than powers of two using padding.
Note that all of our results will also assume a circular padded domain.

\subsection{Convolutions}
\begin{Def}\label{app:def:conv}
	A {\bf convolution} in Setting~\ref{app:set:single} with filter size $k$, dilation $d\in[\lfloor\frac{n-1}{k-1}\rfloor]$, stride $s\in[n-1]$, and channel groups described by a matrix $\*B\in\{0,1\}^{n\times n}$ s.t. $\*B_{[i,j]}=1$ if channels $i$ and $j$ are in the same group and 0 otherwise is a parameterizable operation that for any weight $\*w\in\W$ outputs a function mapping every $\*x\in\X$ to
	\begin{equation}
	\frac1n
	\begin{pmatrix}
	\diag(\atr_s(\underline{\1_{\lceil\frac ns\rceil}}))\sum\limits_{j=1}^c\*B_{[1,j]}\*F_n^{-1}\diag(\*F_n\atr_d(\underline{\*w_{[1,j]}}))\*F_n\*x_{[j]}\\
	\vdots\\
	\diag(\atr_s(\underline{\1_{\lceil\frac ns\rceil}}))\sum\limits_{j=1}^c\*B_{[c,j]}\*F_n^{-1}\diag(\*F_n\atr_d(\underline{\*w_{[c,j]}}))\*F_n\*x_{[j]}
	\end{pmatrix}
	\end{equation}
	where $\*F_n\in\C^{n\times n}$ is the $n\times n$ DFT and $\atr_d:\R^n\mapsto\R^n$ is an atrous permutation of a vector that is equivalent to multiplication by some permutation matrix $\*P_d\in\{0,1\}^{n\times n}$.
	We will use $\Conv_k$ to denote the case of $d=1$, $s=1$, and $\*B=\1_{c\times c}$.
\end{Def}
\begin{Clm}\label{app:clm:conv}
	All multi-channel convolutions of the form given in Definition~\ref{app:def:conv} are contained in the search space of XD-operations of depth $(1,3,1)$.
\end{Clm}
\begin{proof}
	Setting the architecture parameters to be $\*K=\diag(\atr_s(\underline{\1_{\lceil\frac ns\rceil}}))\*F_n^{-1}$, $\*L=\*F_n\*P_d$, $\*M=\*F_n$, $\*b=\0_n$, and $\*C=\*B$, and noting that (a) the DFT and its inverse are both depth 1 K-matrices, (b) multiplying a K-matrix by a diagonal matrix is another K-matrix of the same depth, and (c) permutation matrices are K-matrices of depth 2 yields the result.
	These three facts can be found in the original paper~\citep{dao2020kaleidoscope}.
\end{proof}
\begin{Rem}
	Note that for the case of dilation $d=1$ the result in Claim~\ref{app:clm:conv} holds with depth $\1_3$.
\end{Rem}

\subsection{Parameter-Free Operations}

\begin{Def}\label{app:def:skip}
	The {\bf skip-connection} in Setting~\ref{app:set:single} is parameterizable operation that outputs a function mapping every $\*x\in\X$ to itself.
	The {\bf zero-operation} in Setting~\ref{app:set:single} is parameterizable operation that outputs a function mapping every $\*x\in\X$ to $\0_{c\times n}$.
\end{Def}
\begin{Clm}\label{app:clm:skip}
	The skip-connection and zero-operation are both contained in the search space of XD-operations of depth $\1_3$.
\end{Clm}
\begin{proof}
	For both set the architecture parameters to be $\*K=\*F_n^{-1}$, $\*L=\0_{n\times n}$, $\*M=\*F_n$, and $\*C=\*I_c$.
	To obtain the skip-connection set $\*b=\1_n$;
	to obtain the zero-operation set $\*b=\0_n$.
\end{proof}

\begin{Def}\label{app:def:avgp}
	An {\bf average pooling} operation in Setting~\ref{app:set:single} with filter size $k$, dilation $d\in[\lfloor\frac{n-1}{k-1}\rfloor]$, and stride $s\in[n-1]$ is parameterizable operation outputs a function mapping every $\*x\in\X$ to the output of a convolution (as in Definition~\ref{app:def:conv}) with the same filter size, dilation, and stride, channel groups described by $\*B=\*I_c$, and filters $\*w_{[j,j]}=\1_k/k~\forall~j\in[c]$.
\end{Def}
\newpage
\begin{Clm}\label{app:clm:avgp}
	All average pooling operations are contained in the search space of XD-operations of depth $\1_3$.
\end{Clm}
\begin{proof}
	Setting the architecture parameters to be $\*K=\diag(\atr_s(\underline{\1_{\lceil\frac ns\rceil}}))\*F_n^{-1}$, $\*L=\0_{n\times n}$, $\*M=\*F_n$, $\*b=\atr_d(\underline{\1_k/k})$, and $\*C=\*I_c$ and noting that (a) the DFT and its inverse are both depth 1 K-matrices and (b) multiplying a K-matrix by a diagonal matrix of the same depth is another K-matrix of the same depth yields the result.
\end{proof}

\subsection{Compositions with Multiplication by a Fixed K-Matrix}

\begin{Def}\label{app:def:lin}
	A {\bf fixed linear operation} $\Lin_{\*A}$ in Setting~\ref{app:set:single} with fixed matrix $\*A\in\R^{n\times n}$ is a parameterizable operation that outputs a function mapping every $\*x\in\X$ to $\Lin_{\*A}(\*w)(\*x)=\begin{pmatrix}\*A\*x_{[1]}&\cdots&\*A\*x_{[c]}\end{pmatrix}^T$.
	For example, $\Lin_{\*I_c}=\Id$.
\end{Def}

\begin{Def}\label{app:def:composition}
	Let $\Op_1$ and $\Op_2$ be two parameterizable operations in Setting~\ref{app:set:single} with $\X$.
	Then for any weight $\*w\in\W$ their {\bf composition} $\Op_1\circ\Op_2$ outputs the parameterized function $\Op_1(\*w)\circ\Op_2(\*w)$.
\end{Def}

\begin{Clm}\label{app:clm:composition}
	Let $\Op$ be a parameterizable operation in Setting~\ref{app:set:single} that is contained in the set of XD-operations of some depth $\*d\in\N^3$ and let $\*A$ be a K-matrix of depth $d'$.
	Then $\Op\circ\Lin_{\*A}$ is contained in the set of XD-operations of depth $(\*d_{[1]},\*d_{[2]},\*d_{[3]}+d')$ and $\Lin_{\*A}\circ\Op$ is contained in the set of XD-operations of depth $(\*d_{[1]}+d',\*d_{[2]},\*d_{[3]})$.
\end{Clm}
\begin{proof}
	Let $\*K$ and $\*M$ be the first and last K-matrices of the representation of $\Op$ as an XD-operation, which thus have depth at most $\*d_{[1]}$ and $\*d_{[3]}$, respectively.
	Then the representation of $\Op\circ\Lin_{\*A}$ as an XD-operation is the same except with depth $\*d_{[3]}+d'$ K-matrix $\*M\*A$ as the last K-matrix, and similarly the representation of $\Lin_{\*A}\circ\Op$ as an XD-operation is the same except with depth $\*d_{[1]}+d'$ K-matrix $\*A\*K$ as the first K-matrix.
\end{proof}

\subsection{Other Named Operations}

\begin{Def}
	Suppose we have a fixed $n$-node graph with adjacency matrix $\*A$ and degree matrix $\*D$, and let $\hat{\*A}$ and $\hat{\*D}$ be the adjacency and degree matrices, respectively, of the same graph but with added self-loops.
	Then regular {\bf graph convolution} \citep{kipf2017gcn} in Setting~\ref{app:set:single} with $k=1$ is a parameterizable operation that for any weight $\*W\in\W$ outputs a function mapping every $\*x\in\X$ to $\hat{\*D}^{-\frac12}\hat{\*A}\hat{\*D}^{-\frac12}\*x^T\*w$ and the {\bf diffusion graph convolution} \citep{li2018dcrnn} in Setting~\ref{app:set:single} with $k=1$ is a parameterizable operation that for any weight $\*W\in\W$ outputs a function mapping every $\*x\in\X$ to $\*D^{-1}\*A\*x^T\*w$.
\end{Def}
\begin{Clm}
	Suppose $\*A$ and $\hat{\*A}$ can be represented by K-matrices of depth $d$ and $\hat d$, respectively.
	Then the corresponding graph convolution is contained in the search space of XD-operations of depth $(1,1,\hat d+1)$ and the corresponding diffusion graph convolution in that of depth $(1,1,d+1)$.
\end{Clm}
\begin{proof}
	For any $\*G\in\R^{n\times n}$ we have  $\*G\*x^T\*w=\Lin_{\*G}(\*w)(\*x)\*w=\Conv_1(\*w)(\Lin_{\*G}(\*w)(\*x))=(\Conv_1\circ\Lin_{\*G})(\*w)(\*x)$.
	The result follows by Claims~\ref{app:clm:conv} and~\ref{app:clm:composition}, the fact that a K-matrix multiplied by a diagonal matrix is another K-matrix of the same depth, and by substituting $\*G=\hat{\*D}^{-\frac12}\hat{\*A}\hat{\*D}^{-\frac12}$ (for graph convolution) or $\*G=\*D^{-1}\*A$ (for diffusion graph convolution).
\end{proof}
\begin{Rem}
	Note that the above claim is meaningful because adjacency matrices of realistic graphs are usually sparse and sparse matrices can be efficiently represented as K-matrices \citep{dao2020kaleidoscope}.
\end{Rem}

\begin{Def}
	A {\bf Fourier neural operator} (FNO) \citep{li2021fno} in Setting~\ref{app:set:single} with even $k$ and thus $k/2$ modes is a parameterizable operation that for any weight $\*w\in\W$ outputs a function mapping every $\*x\in\X$ to 
	\begin{equation}
		\begin{pmatrix}
		\Real\left(\sum_{j=1}^c\*F_n^{-1}\diag(\begin{pmatrix}\*w_{[1,j,1:k/2]}+i\*w_{[1,j,k/2+1:k]}&\0_{n-k/2}\end{pmatrix}^T)\*F_n\*x_{[j]}\right)\\
		\vdots\\
		\Real\left(\sum_{j=1}^c\*F_n^{-1}\diag(\begin{pmatrix}\*w_{[c,j,1:k/2]}+i\*w_{[c,j,k/2+1:k]}&\0_{n-k/2}\end{pmatrix}^T)\*F_n\*x_{[j]}\right)
		\end{pmatrix}
	\end{equation}
\end{Def}
\begin{Clm}\label{app:clm:fno}
	The FNO with $k/2$ modes is contained in the search space of XD-operations of depth $(1,4,1)$.
\end{Clm}
\begin{proof}
	Setting the architecture parameters to be $\*K=\*F_n^{-1}$, $\*L\in\C^{n\times n}$ the $n$-sparse matrix mapping $\underline{\*w}$ to $\begin{pmatrix}\*w_{[1,j,1:k/2]}+i\*w_{[1,j,k/2+1:k]}&\0_{n-k/2}\end{pmatrix}^T$, $\*M=\*F_n$, $\*b=\0_n$, and $\*C=\1_{c\times c}$, and noting that an $n$-sparse matrix is a depth-4 K-matrix \citep{dao2020kaleidoscope} yields the result.
\end{proof}
\begin{Rem}
	If we allow the parameter space in Setting~\ref{app:set:single} to be complex then the FNO with all $k$ modes will be contained in the search space of XD-operations of depth $\1_3$.
\end{Rem}

\begin{Def}
	Each channel of {\bf transposed convolution} with stride $d(k-1)+1$, where $k$ is the kernel size and $d$ is the dilation rate, computes a feature map in which each input element is replaced by that element multiplied by the dilated filter of size $d(k-1)+1$.
	The multi-channel extension of this over parameter space $\W=\R^{c\times c\times k}$ is similar to that for standard convolutions.
\end{Def}
\begin{Clm}
	All transposed convolutions with stride equal to the dilated kernel size are contained in the search space of XD-operations of depth $(1,3,3)$.
\end{Clm}
\begin{proof}
	A transposed convolution is equivalent to a regular convolution with the same filter applied to the input after it has been zero-padded and then permuted to separate all entries by $d(k-1)$ zeros.
	Since permutations are K-matrices of depth 2 the result follows by Claims~\ref{app:clm:conv} and Claim~\ref{app:clm:composition}.
\end{proof}

\begin{Def}
	A {\bf depthwise-separable convolution} in Setting~\ref{app:set:single} with filter size $k$ but with parameter space $\W=\R^{c\times k}\times\R^{c\times c}$ is a parameterizable operation that for any weight $\*w\in\W$ outputs $\Conv_1(\*w_{[2]})\circ\Conv_{k,\*I_c}(\*w_{[1]})$, where $\Conv_{k,\*I_c}$ denotes the convolution in Definition~\ref{app:def:conv} with $\*B=\*I_c$.
\end{Def}
\begin{Rem}
	Since both $\Conv_1$ and $\Conv_{k,\*I_c}$ are XD-operations, by definition depthwise-separable convolutions are contained in the search space of composed XD-operations, which by Claim~\ref{app:clm:skip} also contains all of the above operations.
\end{Rem}

%% file: complexity.tex
% !TEX root = main.tex

\section{Practical Complexity of XD-Operations}

\begin{table}[!h]
	\begin{threeparttable}
		\captionof{table}{\label{app:tab:complexity}
			Comparison of the computational and memory costs of XD-operations when substituted for convolutions.
			For simplicity, we consider cases with 2d inputs and where the channel and bias parameters are fixed.
		}
		\begin{tabular}{lcccccccc}
			\hline
			& input & kernel & \multicolumn{2}{c}{minutes / epoch} & \multicolumn{2}{c}{memory (Gb)} & \multicolumn{2}{c}{param. ($\times10^6$) }\\
			Task (backbone) & size & size & $\Conv$ & $\XD$ & $\Conv$ & $\XD$ & $\Conv$ & $\XD$ \\
			\hline
			CIFAR-10 (WRN-40-4) & $32$ & $3$ & 1.4 & 4.3 & 3.73 & 15.6 & 8.96 & 9.08 \\
			Darcy Flow (Conv4$^\ast$) & $85$ & 13 & 0.028 & 0.14 & 4.51 & 5.53 & 0.701 & 0.744  \\
			PSICOV (ResNet-18) & $128$ & 3 & 5.9 & 11 & 1.50 & 10.7 & 0.038 & 0.549  \\
			\hline
		\end{tabular}
		\begin{tablenotes}
			\item[$\ast$] Four-layer convolutional network with parameterized skip (shortcut) connections derived from the FNO network \citep{li2021fno} as described in Section~\ref{sec:pde}.
		\end{tablenotes}
	\end{threeparttable}
\end{table}

In this section we report a detailed comparison of computational costs of the XD-operation compared to a convolution;
this is presented in Table~\ref{app:tab:complexity}.
Due to their familiarity, we present results for tasks that have 2d inputs and thus use 2d convolutions in their default backbone.
Note that since XD-operations are more general than convolutions, they must by definition be at least as expensive as convolutions in both computation and memory.
While in this paper our focus is on absolute performance using learning metrics (e.g. test error), we view finding a good tradeoff between the performance of XD-operations on certain tasks and convolutions, for example by restricting the expressivity of XD-operations, as important directions for future work.

%% file: cifar.tex
% !TEX root = main.tex

\newpage
\section{Experimental Details: CIFAR-10 and Permuted CIFAR-10}

\begin{table}[!h]
	\centering
	\begin{threeparttable}
		\caption{\label{app:tab:cifar}
			Architecture optimizer settings on CIFAR-10 tasks.
			Note that the step-size is updated using the same schedule as the backbone.
		}
		\footnotesize
		\begin{tabular}{cclcccc}
			\hline
			search space & backbone & task & optimizer & initial step-size & warmup epochs & perturb \\
			\hline
			\multirow{4}{*}{$\tilde\Search_\disc$}
			& \multirow{2}{*}{LeNet}
			& CIFAR-10 & Adam & 1E-1 & 0 & 0.1 \\
			&& Permuted & Adam & 1E-1 & 50 & 0.875 \\
			\cline{2-7}
			& \multirow{2}{*}{ResNet-20}
			& CIFAR-10 & Adam & 1E-3 & 0 & 0.1 \\
			&& Permuted & Adam & 1E-1 & 0 & 0.875 \\
			\hline
			\multirow{4}{*}{$\Search_\XD$}
			& \multirow{2}{*}{LeNet}
			& CIFAR-10 & Adam & 1E-4 & 0 & - \\
			&& Permuted & Adam & 1E-3 & 0 & -\\
			\cline{2-7}
			& \multirow{2}{*}{ResNet-20}
			& CIFAR-10 & Adam & 1E-4 & 50 & - \\
			&& Permuted & Adam & 1E-3 & 0 & - \\
			\hline
		\end{tabular}
		%		\begin{tablenotes}
		%			\item[$\ast$] table note 1
		%			\item[$\dagger$] table note 2
		%		\end{tablenotes}
	\end{threeparttable}
\end{table}

For our experiments with image classification backbones we use the standard CIFAR-10 data \citep{krizhevsky2009cifar} and a permuted version where all rows and columns are identically permuted.
For unpermuted data we use standard data augmentation \citep{he2016resnet} while for permuted data we do not use any data augmentation.
As specified in Section~\ref{sec:chrysalis}, we keep the training routine of the model weights the same and tune only the architecture optimizer, the settings of which are specified in Table~\ref{app:tab:cifar}.
Note that for the DARTS operation space we specify a ``perturb'' parameter that specifies how unbiased the initial architecture parameters are towards the backbone operation;
specifically, we initialize architecture parameters so as to assign one minus this quantity as the weight to the backbone operation, so 0.875 means the initialization is uniform (since $|\tilde\Search_\disc|=8$) while 0.1 means the backbone operation is assigned 0.9 of the weight.

\subsection{LeNet}

The LeNet backbone we consider consists of two $\Conv_{5\times 5}$ layers, each followed by $\MaxP_{2\times 2}$, and two fully connected layers.
When warm-starting with XD-operations we use $\AvgP_{2\times2}$ instead of $\MaxP_{2\times2}$, while when warm-starting with the DARTS operations we use $\MaxP_{3\times 3}$.
For the baseline training routine we use 200 epochs of Momentum(0.9), with the first 100 at learning rate 0.01, the next 50 at 0.005, and the last 50 at 0.001.

\subsection{ResNet-20}

We use the implementation and training routine provided here: \url{https://github.com/akamaster/pytorch_resnet_cifar10}.
When replacing operations in the backbone we substitute for both the $\Conv_{3\times3}$ operations and the skip-connections $\Id$;
some of the latter are downsampled, which XD-operations can handle as strides.

\subsection{WideResNet-40-4}

We use the same implementation as for ResNet-20 but adapt the original WRN training routine \citep{zagoruyko2016wideresnet}, except with weight-decay set to $10^{-4}$ (as in ResNet-20);
on the regular CIFAR-10 tasks this does not seem to affect performance.
To conserve computation and memory, we do not tune the architecture optimizer parameters here and simply use the same ones used for ResNet-20;
furthermore, we fix the channel and bias parameters of XD-operations and do not allow the kernel size to be larger the $3\times 3$.
Because of these modifications, we only use our evaluation here as a sanity check for large-network performance of XD-operations and do not include it in the main results.

\newpage
\subsection{DARTS Cell Search}

To search the full DARTS search space, which is a standard NAS benchmark, we use GAEA PC-DARTS, a recent state-of-the-art method \citep{li2021gaea}, using code made available by the authors here: \url{https://github.com/liamcli/gaea_release}.
On CIFAR-10 we simply use their best reported cell but evaluate it using the ``base" routine \citep{yang2020nas}, i.e. without auxiliary losses or additional data augmentation;
this is to obtain fair comparison with the other backbone models.
Note that the model is still much larger and the training routine much more intensive.
On permuted data we follow the standard three-stage pipeline in which we run search four times, train all four found cells and select the best one, and finally train that cell multiple times.

\subsection{DenseNAS Search}

We use the DenseNAS search and evaluation code released by the authors here: \url{https://github.com/JaminFong/DenseNAS}.
While the search space is designed for ImageNet \citep{russakovsky2015imagenet}, we adapt it to CIFAR-10 by taking the DenseNAS-R1 setting and downscale the input sizes to match 32x32 images used.

\begin{table}[!t]
	\centering
	\begin{threeparttable}
			\captionof{table}{\label{app:tab:cifarres}
				Search space comparison on CIFAR-10.
				Validation accuracies are averages of three trials.
			}
			\begin{tabular}{lcccc}
				\hline
				Backbone & Search Space & CIFAR-10 & Permuted$^\ast$ & Cost (hours$^\dagger$) \\
				\hline
				\multirow{3}{*}{LeNet} & backbone & $75.5\pm0.1$ & $43.7\pm0.5$ & 0.3 \\
				& $\tilde\Search_\disc$ & $75.6\pm3.4$ & $47.7\pm1.0$ & 1.0 \\
				& $\Search_\XD$ & $77.7\pm0.7$ & $63.0\pm1.0$ & 0.9 \\
				\hline
				\multirow{3}{*}{ResNet-20} & backbone & $91.7\pm0.2$ & $58.6\pm0.7$ & 0.6 \\
				& $\tilde\Search_\disc$ & $92.7\pm0.2$ & $58.0\pm1.0$ & 5.3 \\
				& $\Search_\XD$ & $92.4\pm0.2$ & $73.5\pm1.6$ & 5.6 \\
				\hline
				\multirow{3}{*}{WRN-40-4} & backbone & $95.2\pm0.1$ & $64.7\pm0.9$ & 4.6 \\
				& $\tilde\Search_\disc$ & $95.2\pm0.2$ & $61.3\pm1.3$ & 19.9 \\
				& $\Search_\XD$ & $95.0\pm0.1$ & $72.9\pm0.8$ & 14.3 \\
				\hline
				ResNet-18 & DenseNAS & $94.5\pm0.3$ & $61.6\pm3.3$ & 3.6  \\
				Cell & DARTS$^\ddagger$ & $96.0\pm0.2$ & $66.3\pm0.5$ & 28.6 \\
				\hline
			\end{tabular}
			\begin{tablenotes}
				\item[$\ast$] No data augmentation used in the permuted case.
				\item[$\dagger$] On a V100 GPU; time for DARTS Cell is training cost only.
				\item[$\ddagger$] Search using GAEA PC-DARTS \citep{li2021gaea}; training using ``base'' routine~\citep{yang2020nas}.
			\end{tablenotes}
		\end{threeparttable}
\end{table}

%% file: apppde.tex
% !TEX root = main.tex

\newpage
\section{Experimental Details: Solving PDEs}

For our PDE experiments, we use the FNO code and setup \citep{li2021fno} provided here: \url{https://github.com/zongyi-li/fourier_neural_operator}. We use the same training routine and settings as the backbone architecture for each task and only tune the architecture optimizer. We consider the following hyperparameters for the architecture optimizer: Adam vs. SGD (with or without momentum), initial learning rate, and number of warmup epochs. The final hyperparameters for each task can be found in Table~\ref{app:tab:pdeopt}. 
Our CNN backbone is analogous to the FNO architecture used for each problem. In particular, the CNN backbone architecture used for each task is simply the FNO architecture where FNO layers of dimension $N$ with $m$ modes are replaced by $N$-dimensional convolutional layers with filters of size $(m+1)^N$ and circular padding to match the dimensionality of FNO. In Table~\ref{app:tab:pde1dres} and Table~\ref{app:tab:pde2dres} we present reported \cite{li2021fno}, reproduced, and our own results on the 1d Burgers' equation and 2d Darcy Flow. 

For AutoDL we use the code and setup provided here: \url{https://github.com/NoamRosenberg/autodeeplab}. 
We only conduct search on the lowest resolution and use the resulting architecture at higher resolutions.
Search was conducted for 40 epochs, as in the original paper, and the search learning rate was tuned.

\begin{table}[!h]
	\centering
	\begin{threeparttable}
		\caption{\label{app:tab:pdeopt}
			Architecture optimizer settings on PDE tasks. Note that the step-size is updated using the same schedule as the backbone. 
		}
		\begin{tabular}{lccc}
			\hline
			task & optimizer & initial step-size  & warmup epochs \\
			\hline
			1d Burgers' equation & Adam & 1E-3 & 0 \\
			1d Burgers' equation (FNO init) & Momentum(0.5) & 1E-4 & 250 \\
			2d Darcy Flow & Momentum(0.5) & 1E-1 & 0 \\
			2d Darcy Flow (FNO init) & Momentum(0.5) & 1E-1 & 0 \\
			2d Navier Stokes ($\nu=10^{-4}, T=30$) & Momentum(0.5) & 5E-3 & 0 \\
			2d Navier Stokes ($\nu=10^{-5}, T=20$) & Momentum(0.5) & 1E-3 & 0 \\
			\hline
		\end{tabular}
	\end{threeparttable}
\end{table}

\begin{table*}[!h]
	\centering
	\begin{threeparttable}
		\caption{\label{app:tab:pde1dres}
			Test relative errors on the 1d Burgers' equation. We were not able to match the FNO-1d results reported by the authors \citep{li2021fno} using their published codebase, however, our proposed XD operations outperform our reproduction of their results at every resolution. Furthermore, we outperform their reported test relative errors on every resolution except $s=4096$, where we roughly match their performance. 
		}
		\begin{tabular}{lcccccr}
					\hline
					Method (source)                      & $s=256$ & $s=512$ & $s=1024$ & $s=2048$ & $s=4096$ & $s=8192$ \\ 
					\hline
					NN   \citep{li2021fno}             &    0.4714    &   0.4561      &   0.4803      &   0.4645      &     0.4779    & 0.4452 \\
					GCN  \citep{li2021fno}              &   0.3999     &    0.4138     &    0.4176     &   0.4157     &   0.4191      &  0.4198 \\
					FCN \citep{li2021fno}               &     0.0958   &    0.1407     &    0.1877     &    0.2313     &   0.2855      &  0.3238 \\
					PCANN  \citep{li2021fno}              &     0.0398   &    0.0395     &   0.0391      &     0.0383    &    0.0392     & 0.0393  \\
					GNO  \citep{li2021fno}              &    0.0555    &   0.0594      &    0.0651     &    0.0663     &      0.0666   & 0.0699  \\
					LNO   \citep{li2021fno}             &    0.0212    &    0.0221     &      0.0217   &    0.0219     &    0.0200     & 0.0189  \\
					MGNO  \citep{li2021fno}              &      0.0243  &     0.0355    &      0.0374   &    0.0360     &    0.0364     & 0.0364  \\
					FNO-1d \citep{li2021fno}              &    0.0149    &     0.0158    &   0.0160      &  0.0146       &   \textbf{0.0142}      &  0.0139 \\
					\hline
					CNN (ours)                &    0.0518    &   0.1220      &   0.1830      &   0.2280      &     0.2730    & 0.2970 \\
					FNO-1d (reproduced)                           &  0.0181      &    0.0191     &    0.0188     &     0.0184    &   0.0183     & 0.0183 \\
					CNN XD (ours) &    \textbf{0.0141}    &     \textbf{0.0079}    &    \textbf{0.0154}     &  \textbf{0.0099}       &  0.0145       & \textbf{0.0123} \\
					FNO-1d XD (ours) &     0.0153   &     0.0154   &   0.0154      &    0.0167     &    0.0160     & 0.0155  \\
					\hline
				\end{tabular}
%		\begin{tablenotes}\footnotesize
%			\item[$\ast$] No data augmentation used in the permuted case.
%			\item[$\dagger$] Training using ``base" routine from \citet{yang2020nas}.
%		\end{tablenotes}
	\end{threeparttable}
\end{table*}

\begin{table*}[!h]
	\centering
	\begin{threeparttable}
		\caption{\label{app:tab:pde2dres}
			Test relative errors on 2d Darcy Flow. Our reproduction of the FNO-2d results outperform those reported by the authors \citep{li2021fno}. Nonetheless, our proposed XD operations outperform both our reproduction and the reported results at every resolution. 
		}
		\begin{tabular}{lccccr}
					\hline
					Method (source)                      & $s=85$ & $s=106$ & $s=141$ & $s=211$ & $s=421$ \\ 
					\hline
					NN \citep{li2021fno}                &     0.1716   &     -    &    0.1716     &    0.1716    &   0.1716  \\
					GCN \citep{li2021fno}               &     0.0253   &     -    &   0.0493      &    0.0727     &       0.1097    \\
					FCN \citep{li2021fno}               &   0.0299     &    -     &   0.0298      &     0.0298    &     0.0299      \\
					PCANN \citep{li2021fno}               &    0.0244    &    -     &    0.0251     &    0.0255     &    0.0259       \\
					GNO \citep{li2021fno}               &    0.0346    &    -     &    0.0332     &    0.0342     &       0.0369    \\
					LNO \citep{li2021fno}               &    0.0520    &    -     &     0.0461    &     0.0445    &     -     \\
					MGNO \citep{li2021fno}               &    0.0416    &     -    &    0.0428     &   0.0428      &      0.0420    \\
					FNO-2d \citep{li2021fno}               &    0.0108    &     -    &    0.0109     &      0.0109   &    0.0098      \\
					\hline
					CNN (ours)                 &     0.0404   &    0.0495     &   0.0613      &    0.0813     &  0.1150       \\
					FNO-2d (reproduced)                           &     0.0096   &    0.0092     &    0.0091     &     0.0091    &    0.0091     \\
					CNN XD (ours)  &    \textbf{0.0065}    &     \textbf{0.0065}    &    \textbf{0.0065}     &  \textbf{0.0071}       &  \textbf{0.0066} \\
					FNO-2d XD (ours)  &    0.0082    &     0.0079   &      0.0077   &      0.0076   &    0.0074  \\
					\hline
				\end{tabular}
%		\begin{tablenotes}\footnotesize
%			\item[$\ast$] No data augmentation used in the permuted case.
%			\item[$\dagger$] Training using ``base" routine from \citet{yang2020nas}.
%		\end{tablenotes}
	\end{threeparttable}
\end{table*}

\begin{figure}[!h]
	\includegraphics[width=\linewidth]{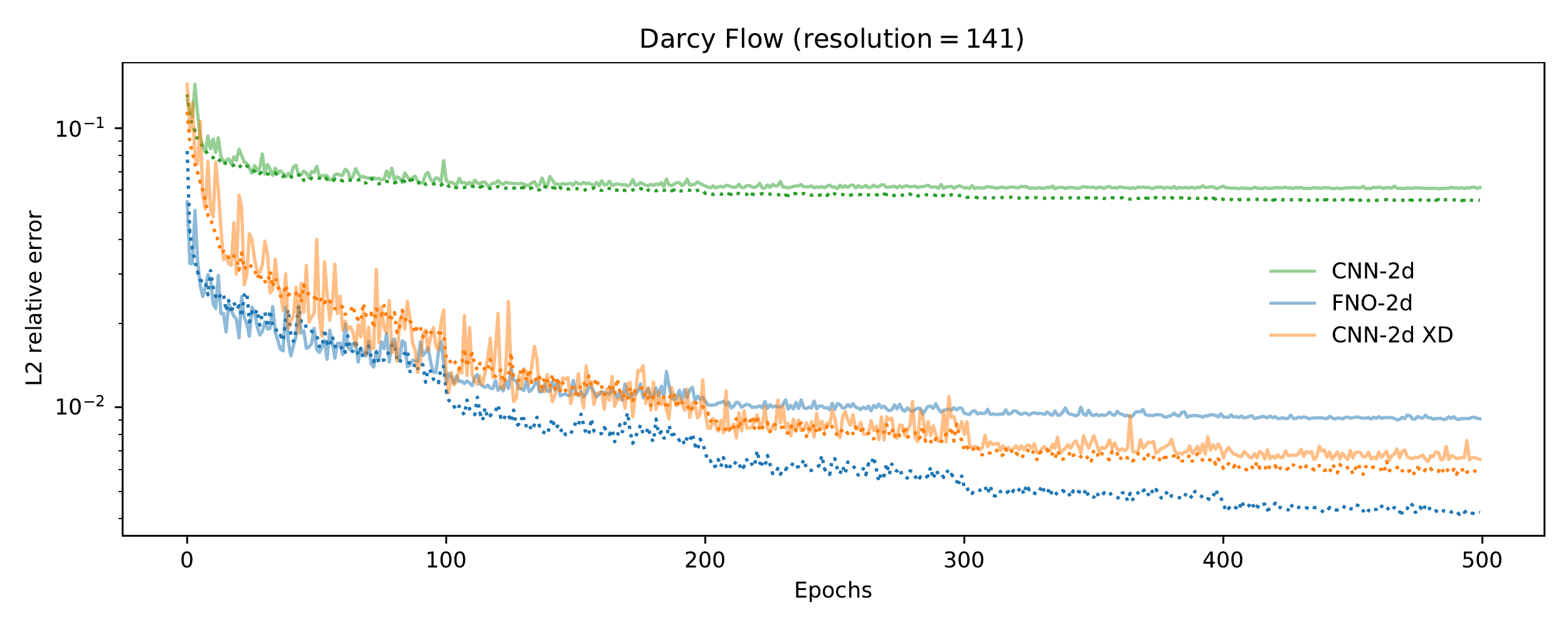}
	\caption{Training curves (dotted) and test curves (solid) on Darcy Flow at resolution 141, showing better generalization of XD-operations.}
\end{figure}

%% file: approtein.tex
% !TEX root = main.tex

\newpage
\section{Experimental Details: Protein Folding}

\begin{table}[!h]
	\centering
	\begin{threeparttable}
		\caption{\label{app:tab:proteinopt}
			Architecture optimizer settings on for our protein folding experiments, across different ResNet depths. Note that the same step-size is used throughout since the backbone has no step-size schedule. 
		}
		\begin{tabular}{lccc}
			\hline
			search space & optimizer & step-size  & warmup epochs \\
			\hline
			ResNet-4 XD & Adam & 1E-4 & 2 \\
			ResNet-6 XD & Momentum(0.99) & 1E-4 & 2 \\
			ResNet-10 XD & Momentum(0.99) & 1E-3 & 2 \\
			ResNet-18 XD & Momentum(0.9) & 5E-4 & 2 \\
			ResNet-34 XD & Momentum(0.9) & 5E-4 & 2 \\
			\hline
		\end{tabular}
	\end{threeparttable}
\end{table}

\begin{table}[!h]
	\centering
	\begin{threeparttable}
		\caption{\label{app:tab:proteinperf}
			Test MAE$_8$ of the Dilated ResNet of \cite{adhikari2020}, compared to a standard ResNet backbone and XD-operations applied to ResNet. Results are averaged over 3 trials. 
		}
		\begin{tabular}{lccccc}
			\hline
			 Method & depth $ = 4$ & depth $ = 6$ & depth $ = 10$ & depth $ = 18$ & depth $ = 34$ \\
			 \hline
			 ResNet & $5.99\pm0.43$ & $5.30\pm0.11$ & $4.91\pm0.25$ & $4.80\pm0.07$ & $4.66\pm0.15$ \\
			 Dilated ResNet & $6.04\pm0.33$ & $5.49\pm0.02$ & $4.64\pm0.08$ & $4.59\pm0.22$ & $4.50\pm0.13$ \\
			 ResNet XD & ${\bf5.59\pm0.09}$ & ${\bf4.59\pm0.17}$ & ${\bf4.25\pm0.16}$ & ${\bf4.22\pm0.03}$ & ${\bf4.00\pm0.07}$ \\
			\hline
		\end{tabular}
	\end{threeparttable}
\end{table}

For our protein folding experiments, our code is a PyTorch re-implementation of the PDNET code and setup \cite{adhikari2020} provided here: \url{https://github.com/ba-lab/pdnet}. As before, we use the same training routine and settings as the Dilated ResNet architecture and only tune the architecture optimizer. We consider the following hyperparameters for the architecture optimizer: Adam vs. SGD (with or without momentum), learning rate, and number of warmup epochs. The final hyperparameters for each depth can be found in Table~\ref{app:tab:proteinopt}. Our ResNet backbone differs from Dilated ResNet in that its dilation rate is set to 1 in every convolutional layer. In Table~\ref{app:tab:proteinperf}, we present average MAE$_8$ on the PSICOV test set for each method at each depth. 

%% file: appseq.tex
% !TEX root = main.tex

\section{Experimental Details: Music Modeling and Sequence Modeling}

\begin{table}[!h]
	\centering
	\begin{threeparttable}
		\caption{\label{app:tab:seq}
			Architecture optimizer settings on sequence modeling tasks.
			Note that the step-size is updated using the same schedule as the backbone.
		}
		\begin{tabular}{lccc}
			\hline
			task & optimizer & initial step-size & warmup epochs \\
			\hline
			Permuted MNIST & Adam & 2E-4 & 0 \\
			JSB Chorales & Adam & 2E-4 & 25 \\
			Nottingham & Adam & 2E-3 & 0 \\
			Penn Treebank & Adam & 2E-6 & 0 \\
			\hline
		\end{tabular}
		%		\begin{tablenotes}
		%			\item[$\ast$] table note 1
		%			\item[$\dagger$] table note 2
		%		\end{tablenotes}
	\end{threeparttable}
\end{table}

\begin{table*}[!h]
	\centering
	\begin{threeparttable}
		\caption{\label{app:tab:seqres}
			XD-operations applied to TCNs compared to recent empirical results in sequence modeling.
			Our results are averages of three trials.
			Methods achieving within one deviation of the best performance are {\bf bolded}.
		}
		\footnotesize
		\begin{tabular}{lcccc}
			\hline
			& Permuted MNIST$^\ast$ & JSB Chorales & Nottingham & Penn Treebank \\
			Method (source)  & (error) & (loss) & (loss) & (perplexity) \\
			\hline
			LSTM \citep{bai2018tcn} & 14.3 & 8.45 & 3.29 & 78.93 \\
			GRU \citep{bai2018tcn} & 12.7 & 8.43 & 3.46 & 92.48 \\
			RNN \citep{bai2018tcn} & 74.7 & 8.91 & 4.05 & 114.50 \\
			TCN backbone \citep{bai2018tcn} & 2.8 & 8.10 & 3.07 & 88.68 \\
			TrellisNet \citep{bai2019trellis} & 1.87 & - & - & {\bf 54.19} \\
			R-Transformer \citep{wang2020rtransformer} & - & - & {\bf 2.37} & 84.38 \\
			HiPPO-LegS \citep{gu2020hippo} & {\bf 1.7} & - & - & - \\
			\hline
			TCN backbone (reproduced) & $2.89\pm0.04$ & $8.17\pm0.01$ & $2.97\pm0.01$ & $88.49\pm0.31$ \\
			TCN backbone XD (ours) & ${\bf 1.75\pm0.11}$ & ${\bf 8.07\pm0.02}$ & $2.81\pm0.05$ & $84.11\pm0.25$ \\
			Undilated TCN (ours) & $11.3\pm2.1$ & $8.16\pm0.04$ & $3.21\pm0.02$ & $94.30\pm0.33$ \\
			Undilated TCN XD (ours) & ${\bf 1.77\pm0.10}$ & ${\bf 8.07\pm0.01}$ & $2.84\pm0.02$ & $85.04\pm0.49$ \\
			\hline
		\end{tabular}
		\begin{tablenotes}\footnotesize
			\item[$\ast$] We use depth $\*d=(3,3,3)$ XD-operations for permuted MNIST experiments;
			elsewhere we use~$(1,3,1)$.
			Results within a standard deviation of the best are {\bf bolded}.
		\end{tablenotes}
	\end{threeparttable}
\end{table*}

For our sequence modeling experiments we use the TCN code \citep{bai2018tcn} provided here: \url{https://github.com/locuslab/TCN}.
As before we use the same settings and training routine as the backbone for all tasks, tuning only the architecture optimizer.
The specific settings are provided in Table~\ref{app:tab:seq}.
For both the baselines and XD-operations we use the same optimizer settings for both the dilated and undilated TCN backbones.
In Table~\ref{app:tab:seqres} we present results for both music modeling and for two additional benchmarks---permuted MNIST and Penn Treebank---on which we see a similar pattern of XD-operations being able to recover and even beat (dilated) TCN performance starting from an undilated network.